%% file: main.tex
\documentclass[twoside]{article}
\input{math_commands}

\usepackage[utf8]{inputenc} % allow utf-8 input
\usepackage[T1]{fontenc}    % 

\usepackage[capitalize,noabbrev]{cleveref}

\usepackage[accepted]{aistats2023}
\usepackage[round]{natbib}
\renewcommand{\bibname}{References}
\renewcommand{\bibsection}{\subsubsection*{\bibname}}

\begin{document}

\twocolumn[

\aistatstitle{Blessing of Class Diversity in Pre-training}

\aistatsauthor{ Yulai Zhao \And Jianshu Chen \And  Simon S. Du }

\aistatsaddress{ Princeton University \And  Tencent AI Lab \And University of Washington} ]

\begin{abstract}
\input{abs}

\end{abstract}
\section{INTRODUCTION}
\label{sec:intro}
\input{intro}

\section{RELATED WORK}
\label{sec:related_work}
\input{related_work.tex}

\section{PRELIMINARIES}
\label{sec:preliminary}
\input{preliminary}

\section{MAIN RESULTS}
\label{sec:main_thm}
\input{main_thm}
\input{setting1}
\input{setting2}

\section{CONCLUSION AND FUTURE WORK}
\label{sec:con}
\input{conclusion.tex}

\subsubsection*{Acknowledgements}
This work was supported in part by NSF CCF 2212261, NSF IIS 2143493, NSF DMS-2134106, NSF CCF
2019844, NSF IIS 2110170, and a gift funding from Tencent.

\bibliography{references}
\bibliographystyle{abbrvnat}

\newpage
\appendix

\onecolumn
\section{TECHNICAL PROOFS} 
\label{sec:pf}
\input{proof.tex}
\input{pf_linear.tex}
\input{pf_DNN}

\section{EXPERIMENTS}
\label{sec:exp}

\input{exp_new.tex}

\input{exp_supply.tex}

\vfill

\end{document}

%% file: math_commands.tex
\usepackage{mathtools}

\usepackage{algorithm}
\usepackage{algorithmic}
\usepackage{amsmath,amssymb,amsthm,amsfonts,bm,bbm, nicefrac}

\usepackage{tabularx}
\usepackage[svgnames,x11names,table]{xcolor}

\allowdisplaybreaks
\usepackage{graphicx}
\usepackage{subfigure}
\usepackage{booktabs} 
\usepackage[hyphens]{url}
\usepackage{hyperref}
\usepackage[hyphenbreaks]{breakurl}

 \hypersetup{ 
    breaklinks=true,
    colorlinks=true,
    linkcolor=DarkBlue,
    citecolor=DarkBlue,
    filecolor=DarkBlue,
    urlcolor=DarkBlue,
}

\def\eqref#1{equation~\ref{#1}}
\def\1{\bm{1}}

\DeclareMathAlphabet{\mathsfit}{\encodingdefault}{\sfdefault}{m}{sl}
\SetMathAlphabet{\mathsfit}{bold}{\encodingdefault}{\sfdefault}{bx}{n}

\def\gF{{\mathcal{F}}}

\def\gH{{\mathcal{H}}}

\def\gJ{{\mathcal{J}}}

\def\gN{{\mathcal{N}}}
\def\gP{{\mathcal{P}}}
\def\gQ{{\mathcal{Q}}}

\def\gW{{\mathcal{W}}}
\def\gX{{\mathcal{X}}}

\def\sR{{\mathbb{R}}}

\newcommand{\E}{\mathbb{E}}

\DeclareMathOperator*{\argmin}{arg\,min}

\newenvironment{enumerate*}%
{\begin{enumerate}[leftmargin=*,topsep=0pt]%
		\setlength{\itemsep}{0pt}%
		\setlength{\parskip}{0pt}}%
	{\end{enumerate}}

\newcommand{\pre}{\mathrm{p}}
\newcommand{\down}{\mathrm{d}}

\theoremstyle{plain}
\newtheorem{theorem}{Theorem}[section]
\newtheorem{proposition}[theorem]{Proposition}
\newtheorem{lemma}[theorem]{Lemma}
\newtheorem{corollary}[theorem]{Corollary}
\theoremstyle{definition}
\newtheorem{definition}[theorem]{Definition}
\newtheorem{assumption}[theorem]{Assumption}
\theoremstyle{remark}
\newtheorem{remark}[theorem]{Remark}

\def\showcomments{1}
\ifnum\showcomments=1
\newcommand{\simon}[1]{[\textcolor{red}{Simon: #1}]}
\newcommand{\yulai}[1]{[\textcolor{blue}{Yulai: #1}]}
\newcommand{\jianshu}[1]{[\textcolor{purple}{Jianshu: #1}]}
\else
\newcommand{\simon}[1]{}
\newcommand{\yulai}[1]{}
\newcommand{\jianshu}[1]{}
\fi

%% file: abs.tex
This paper presents a new statistical analysis aiming to explain the recent superior achievements of the pre-training techniques in natural language processing (NLP).
We prove that when the classes of the pre-training task (e.g., different words in the masked language model task) are sufficiently diverse, in the sense that the least singular value of the last linear layer in pre-training (denoted as $\tilde{\nu}$) is large, then pre-training can significantly improve the sample efficiency of downstream tasks.
Specially, we show the transfer learning excess risk enjoys an $O\left(\frac{1}{\tilde{\nu} \sqrt{n}}\right)$ rate, in contrast to the $O\left(\frac{1}{\sqrt{m}}\right)$ rate in the standard supervised learning. Here, $n$ is the number of pre-training data and $m$ is the number of data in the downstream task, and typically $n \gg m$.
Our proof relies on a vector-form Rademacher complexity chain rule for disassembling composite function classes and a modified self-concordance condition.
These techniques can be of independent interest.

%% file: intro.tex
Pre-training refers to training a model on a few or many tasks to help it learn parameters that can be used in other tasks.
For example, in natural language processing (NLP), one first pre-trains a complex neural network model to predict masked words (masked language modeling), and then fine-tunes the model on downstream tasks, e.g., sentiment analysis~\citep{devlin2018bert}.

Recently, the pre-training technique has
revolutionized the NLP area. 
Models based on this technique have dramatically improved the performance for numerous downstream tasks~\citep{devlin2018bert, radford2018improving, yang2019xlnet, clark2020electra, lan2019albert, liu2019roberta}. 

Despite the large body of empirical work on pre-training, satisfactory theories are still lacking, especially theories that can explain the success of pre-training in NLP.
Existing theories often rely on strong distributional assumptions~\citep{lee2020predicting}, smoothness conditions~\citep{robinson2020strength} or noise-robustness conditions~\citep{bansal2020self} to relate the pre-training task(s) to downstream tasks.
These assumptions are often hard to verify.

A line of work studied \emph{multi-task} pre-training~\citep{caruana1997multitask,baxter2000model,maurer2016benefit,du2020few,tripuraneni2020provable,tripuraneni2020theory,thekumparampil2021sample}.
In particular, recently, researchers have identified a new condition, the \emph{diversity} of pre-training tasks, which has been shown to be crucial to allowing pre-trained models to be useful for downstream tasks.
See Section~\ref{sec:related_work} for more detailed discussions on related work.

Unfortunately, this line of theory cannot be used to explain the success of pre-training in NLP.
The theory of multi-task pre-training requires \emph{a large number of diverse tasks}, e.g., the number of tasks needs to be larger than the last layer's input dimension (a.k.a. embedding dimension), which is typically 768, 1024, or 2048~\citep{devlin2018bert}.
However, in NLP pre-training, there are only a few, if not one, pre-training tasks.
Therefore, we need a new theory that applies to this setting.

Since in NLP pre-training, we do not have multiple tasks, we propose to study \emph{the blessing of multiple classes}.
Concretely, consider the Masked Language Model (MLM) pre-training task in NLP.
In such a pre-training task, we have a large collection of sentences (e.g. from Wikipedia).
During the pre-training phase, we randomly mask a few words in each sentence
% \jianshu{@Simon @Yulai: In BERT pretraing, it randomly masks 15\% of the words in each sentence. We may need to change this statement to ``a few words in each sentence''?} 
and predict the masked words using the remaining words in this sentence.
This pre-training task is a \emph{multi-class classification problem} where the number of classes is about 30K when using byte-pair-encoding (BPE) sub-word units.\footnote{This is a standard setting in the BERT model~\citep{devlin2018bert} and is widely adopted as a common practice. By breaking down the English words into BPE sub-word units, it could drastically increase the coverage of the English language by using a relatively small (32768) vocabulary.} Note that this number is much larger than the embedding dimension (768, 1024, or 2048).

In this paper, we develop a new statistical analysis aiming to explain the success of pre-training for NLP.
The key notion of our theory is the \emph{diversity of classes}, which serves a similar role as the \emph{diversity of tasks} in multi-task pre-training theory~\citep{du2020few,tripuraneni2020provable}.
% Our theory is not only applicable to the real-world NLP pre-training setting, but also inspires new techniques to improve the practical performance.
We summarize our contributions below.

First, we define a new notion, \emph{diversity of classes}, which is the least singular value of the last linear layer in pre-training.
We prove finite-sample bounds to show that for the cross-entropy loss, if the diversity of classes is large, then pre-training on a single task provably improves the statistical efficiency of the downstream tasks.
We give concrete bounds on linear representation and deep neural networks to showcase our general theoretical results.
To our knowledge, this is the first set of theoretical results that demonstrates the statistical gain of the standard practice of NLP pre-training, without strong distributional or smoothness conditions.

Second, from a technical point of view, previous theoretical work on multi-task learning~\citep{du2020few, tripuraneni2020theory} builds on scalar output and thus could not apply to multi-class tasks (e.g., cross-entropy loss). 
We introduce a vector-form Rademacher complexity chain rule for disassembling composite function classes based on vector-form Rademacher contraction property~\citep{maurer2016vector}. 
This generalizes the scalar-form chain rule in~\citet{tripuraneni2020theory}.
Furthermore, we adopt the \emph{modified self-concordance} condition to show that the least singular value of the last linear layer serves as a diversity parameter for cross-entropy loss.
We believe our techniques can be useful in other problems.

% Third, inspired by our theory, we develop a new regularization technique to promote class diversity for multi-class pre-training.
% We apply the negative log determinant regularizer only to the last linear layer of the pre-training model in masked language modeling.
% Our empirical results on BERT-base show this technique can boost the performance of downstream tasks.

%\end{enumerate}

\paragraph{Organization.}
This paper is organized as follows.
In Section~\ref{sec:related_work}, we review the related work.
In Section~\ref{sec:preliminary}, we formally describe the problem setup and introduce the necessary definitions.
In Section~\ref{sec:main_thm}, we state our main Theorem~\ref{thm_main} then instantiate it with several settings. 
We conclude and discuss some interesting future directions in Section~\ref{sec:con}. All proofs are deferred to Appendix~\ref{sec:pf}. 
In Appendix~\ref{sec:exp}, we present some preliminary empirical results on how our theory inspires new regularization techniques.

%% file: related_work.tex
%%\textbf{Multitask learning}
% \citet{du2020few} and \citet{tripuraneni2020theory} focus on transfer learning where some tasks are given large amount of data for learning a good representation. 
% This representation is beneficial for goal task where data points are few. We briefly introduce their formulations behind. Conceptually these two papers are similar. Based on these 2 works, we try to adapt to \cite{saunshi2020mathematical} which shows effectiveness of linear classification tasks.

%\subsection{Prior Work}
Here we mostly focus on the theoretical aspects of pre-training.
While there is a long list of work demonstrating the empirical success of self-supervised learning, there are only a few papers that study its theoretical aspects.
One line of work studied the theoretical properties of \emph{contrastive learning}~\citep{saunshi2019theoretical,tosh2020contrastive}, which is a different setting considered in this paper.
% because we focus on the pre-training met which is most relevant to NLP
The most relevant one is by \citet{lee2020predicting} which showed that if the input data and pre-training labels were independent (conditional on the downstream labels), then pre-training provably improved statistical efficiency.
%use $\phi(X_1)$ to predict $X_2$ (instead of $W_1\phi(X_1)$), and show if $X_2$ and $X_1$ are independent given $Y$, then self-supervised learning provably improves statistical efficiency.
However, this conditional independence assumption rarely holds in practice.
	For example, in the question-answering task, this assumption implies that given the answer, the question sentence and the masked word are independent.
~\citet{robinson2020strength} assumed the Central Condition and a smoothness condition that relates the pretraining task and the downstream task.
\citet{bansal2020self} related generalization error of self-supervised learning to the noise-stability and rationality.
However, it is difficult to verify the assumptions in these papers.

A recent line of theoretical work studied multi-task pre-training~\citep{du2020few,tripuraneni2020provable,tripuraneni2020theory,thekumparampil2021sample} in which a notion, diversity, has been identified to be the key that enables pre-training to improve statistical efficiency. Experiments also supported the idea that increasing the diversity of the training data helps generalization~\citep{zhang2022unveiling}.

Theories on multi-task pre-training generally require a large number of diverse tasks, and thus are not applicable to NLP, as we have mentioned. In comparison, we study single-task multi-class pre-training which is different from theirs. \citet{du2020few} noted that their results allowed an easy adaptation to multi-class settings (see Remark 6.2 therein). However, they only focused on quadratic loss with one-hot labels for multi-class classification. Instead, we study the commonly used cross-entropy loss.
 
While their analyses do not imply results in our setting, our theoretical analyses are inspired by this line of work.

%The major drawback in existing theoretical work is that they cannot explain crucial questions in self-supervised learning applied to NLP: how the neural network architecture and the self-supervised signals affect the statistical efficiency.
%Furthermore, existing work cannot produce new ways for improving the statistical or computational efficiency.

%\simon{TO DO: add diversity regularizer related work.}
% Our paper uses a diversity regularizer proposed in~\citet{zou2012priors} to improve the performance of pre-training.
% We note that there are other diversity regularizers~\citep{xie2017uncorrelation,mariet2015diversity,cogswell2015reducing}.
% These may also improve the performance as the one in~\citet{zou2012priors}.
% We leave it as future work to investigate these regularizers.

%% file: preliminary.tex
In this section, we introduce the necessary notations, the problem setup, and several model-dependent quantities used in pre-training and downstream task learning.
% needed for our theoretical study by adapting formulation in \cite{du2020few, tripuraneni2020theory} to realistic language models. Our algorithm consists of two stages: pre-training and downstream finetuning. This is the ordinary algorithm scheme objective for transfer learning and widely used in applying natural language models.
%\simon{TO Polish}

\subsection{Notations and Setup}
\label{sec:notations}

\textbf{Notations} Let $[n] = \{1, 2, \cdots , n\}$. We use $\|\cdot\|$ or $\|\cdot \|_2$
to denote the $\ell_2$ norm of a vector. Let $\gN(\mu, \sigma_2)$ be the one-dimensional Gaussian distribution. For a matrix $\mathbf{W} \in \sR^{m \times n}$, let $\|\mathbf{W}\|_{1, \infty} = \max_{q} (\sum_{p} |\mathbf{W}_{q,p}| )$ and $\|\mathbf{W}\|_{\infty \to 2}$ be the induced $\infty$-to-$2$ operator norm. We use the standard $O(\cdot), \Omega(\cdot)$ and $\Theta(\cdot)$ notation to hide universal constant factors, and use $\widetilde{O}(\cdot)$ to hide logarithmic factors. We also use
$a \lesssim b$ to indicate $a = O(b)$.

\textbf{Problem setup}
This work is in line with previous transfer learning theories~\citep{du2020few,tripuraneni2020theory} that first pre-train on a large corpus to get a good representation, which, could be future utilized by various downstream tasks.
Formally, the procedure is divided into two stages: the pre-training stage to find a representation function and the downstream training stage to obtain a predictor for the downstream task. 
% \yulai{
In both stages, we use $\hat{R}$ to represent empirical risk and use $R$ to represent expected loss.
% }

In the first stage, we have one pre-training task with $n$ samples, $\{x_i^{\pre},y_i^{\pre}\}_{i=1}^n$, where $x_i^{\pre} \in \gX^\pre \subset \mathbb{R}^d$ is the input and $y_i^{\pre} \in \{0,1\}^{k-1}$ is the one-hot label for $k$-class classification (if $y_i^{\pre}$ is all-zero then it represents the $k$-th class).\footnote{We assume only one pre-training task for the ease of presentation. It is straightforward to generalize our results to multiple pre-training tasks.}
For instance, in masked language modeling, the input of each sample is a sentence with one word masked out, and the label is the masked word.\footnote{Here we say only one word being masked only for the ease of presentation. It is straightforward to generalize our results to the case where multiple words are masked out.}
% \jianshu{@Simon @Yulai: shall we clarify (in a footnote?) that masking multiple words (in practice) is equivalent to having multiple samples each of which has one target to predict? Otherwise, reviewers may bring up this gap between theory and practice.}
$k$ in this example is the size of the vocabulary ($\approx 30K$).
%(e.g., contexts in NLP) combining with the label in if supervised or without label in self-supervised learning. 
%We refer to $x \in \sR^d$ as samples. 
We aim to obtain a good representation function $\hat{h}$ within a function class $\gH \subset \{\mathbb{R}^d \rightarrow \mathbb{R}^r\}$ where $r$ is the embedding dimension (often equals to 768, 1024, 2048 in NLP pre-training). For example, one popular choice of the representation function $\hat{h}$ in NLP applications is the Transformer model and its variants~\citep{vaswani2017attention,devlin2018bert}.
On top of the representation, we predict the labels using function $f^\pre$ within function class $\mathcal{F}^{\pre} \subset \{\mathbb{R}^r \rightarrow \mathbb{R}^{k-1}\}$.
%At the top of representation we apply a layer for classification. 
%\begin{equation}
%	\gF = \{ f| f(z) \in \sR^{ k-1} \}
%\end{equation}
%where $k$ is the number of classes of pre-training or size of vocabulary( $k=30522$ for BERT).  

To train the representation function and predictor in pre-training stage, we consider the Empirical Risk Minimization (ERM) procedure
\begin{align*}
	\hat{h} &= \argmin_{h \in \gH} \min_{f^{\pre} \in \gF^{\pre}} \hat{R}_{\pre}(f^{\pre},h)\\
	&\triangleq \argmin_{h \in \gH} \min_{f^{\pre}\in \gF^{\pre}} \frac{1}{n} \sum_{i=1}^n \ell(f^{\pre} \circ h(x_{i}^{\pre}), y_{i}^{\pre})
\end{align*}
where $\ell$ is the loss function.
We overload the notation for both the pre-training task and the downstream task, i.e., for pre-training, $\ell: \mathbb{R}^{k-1} \times \{0,1\}^{k-1} \rightarrow \mathbb{R}$ and for the downstream task, $\ell: \mathbb{R}^{k'-1} \times \{0,1\}^{k'-1} \rightarrow \mathbb{R}$.
e.g., cross-entropy: 
% \simon{@yulai: add formula here.}\yulai{added}
$\ell(\hat{y} ; y) =  - y^\top \hat{y} + \log{(1 + \sum_{s=1}^{k-1} \exp\left(\hat{y}_s\right) )}$.

    Now for the downstream task, we assume there are $m$ samples $\{x_i^{\down},y_i^{\down}\}_{i=1}^m$. Note that $x_i^{\down} \in \gX^\down \subset \mathbb{R}^d$ is the input and $y_i^{\down} \in \{0,1\}^{k'-1}$ is the one-hot label for $k'$-class classification.\footnote{For simplicity, we assume we only have one downstream task. Our theoretical results still apply if we have multiple downstream tasks.}
 Note that in most real-world applications, we have $n \gg m$ and $k \gg k^\prime$. For example, in sentiment analysis, $k' = 2 $ (``positive" or ``negative"). A widely studied task SST-2~\citep{wang2018glue} has $m \approx 67K$, which is also generally much smaller than the pre-training corpus (e.g., $n>100$M samples).
%  \simon{What is $m$ in sentiment analysis?} \yulai{added GLUE SST-2}
%  in which people only care about ``positive'' or ``negative'' as judgment of a sentence, thus $k^\prime=2$. 
%Analogously, downstream task $f_0$ has $m$ samples along with $(k^\prime - 1)$-onehot labels. Again it means a $k^\prime$-class classification task. The classifier layer is similarly defined as
%\begin{equation}
%	\gF^{\pre}} = \{ f| f(z) \in \sR^{k^\prime-1}\}
%\end{equation}

For the downstream task, we fix the representation function learned from the pre-training task and train the task-dependent predictor within $\mathcal{F}^{\down} \subset \{\mathbb{R}^r \rightarrow \mathbb{R}^{k'-1}\}$:
\begin{align*}
	\hat{f}^{\down} &= \argmin_{f^{\down} \in \gF^{\down}} \hat{R}_{\down}(f^{\down}, \hat{h})\\
	&\triangleq \argmin_{f^{\down} \in \gF^{\down}} \frac{1}{m} \sum_{i=1}^m \ell (f^{\down} \circ \hat{h}(x^{\down}_{i}), y^{\down}_{i}).
\end{align*}

Therefore, our predictor for the downstream task consists a pair $(\hat{f}^{\down}, \hat{h})$.
We use the following risk to measure the performance of predictor and representation
\begin{align*}
  &\text{Transfer~Learning~Risk} \triangleq\\
 &\qquad \qquad \qquad R_{\down}(
	\hat{f}^{\down}, \hat{h}) - 	\mathop{\mathbb{E}}\limits_{x^{\down},y^{\down}} \left[\ell\left(g^{\down}\left(x^{\down}\right),y^{\down}\right)\right]
\end{align*}
where we define 
 $$R_{\down}(
 \hat{f}^{\down}, \hat{h}) \triangleq \mathop{\mathbb{E}}\limits_{x^{\down},y^{\down}}\left[\ell\left(\hat{f}^{\down}\circ \hat{h}\left(x^{\down}\right),y^{\down}\right)\right]$$ 
 as the expected loss (the expectation is over the distribution of the downstream task), and   $$g^{\down} =	\argmin_{g \in \{\mathbb{R}^d \rightarrow \mathbb{R}^{k'-1}\}}\mathbb{E}_{x^{\down},y^{\down}} \left[\ell\left(g\left(x^{\down}\right),y^{\down}\right)\right]$$
 is the optimal predictor for the downstream task.
 
 In our analysis, we also need to use the following term to characterize the quality of pre-training
 \begin{align*}
 \text{Pre-training~Risk} \triangleq R_{\pre}(
 \hat{f}^{\pre}, \hat{h}) - \mathop{\mathbb{E}}\limits_{x^{\pre},y^{\pre}} \left[\ell\left(g^\pre \left(x^{\pre}\right),y^{\pre}\right)\right],
 \end{align*}
 where
 $$R_{\pre}(
 \hat{f}^{\pre}, \hat{h}) \triangleq \mathop{\mathbb{E}}\limits_{x^{\pre},y^{\pre}}\left[\ell\left(\hat{f}^{\pre}\circ \hat{h}\left(x^{\pre}\right),y^{\pre}\right)\right]$$ is the expected loss, and   
 $$ g^{\pre}	= \argmin_{g \in \{\mathbb{R}^d \rightarrow \mathbb{R}^{k-1}\}}\mathbb{E}_{x^{\pre},y^{\pre}} \left[\ell\left(g\left(x^{\pre}\right),y^{\pre}\right)\right]$$
 is the optimal predictor for the pre-training task.

Following the existing work on representation learning~\citep{maurer2016benefit,du2020few,tripuraneni2020theory}, throughout the paper, we make the following realizability assumption, which is also a standard assumption in the classical PAC learning framework~\citep{shalev2014understanding}.
\begin{assumption}[Realizability] \label{assump_realizability}
	There exist $h \in \gH$, $f^{\pre} \in \gF^{\pre}$, $f^{\down} \in \gF^{\down}$ such that $g^{\pre} = f^{\pre}\circ h$ and $g^{\down} = f^{\down}\circ h$.
\end{assumption}
This assumption posits that the representation class and the task-dependent prediction classes are sufficiently expressive to contain the optimal functions.
Importantly, the pre-training and downstream tasks share a \emph{common} optimal representation function $h$.
This assumption formalizes the intuition that pre-training learns a good representation that is also useful for downstream tasks.

As for the setting that is of most interest to NLP pre-training, where
the loss function ` is cross-entropy, $\gF^\pre$ and $\gF^\down$ are sets of linear functions, we make the following assumption on both pre-training and downstream tasks to describe how the underlying data are generated.
\begin{assumption}[Multinomial Logistic Data] \label{assump_samples}
	For a $K$-class classification task with $q$ samples, $\{ x_i, y_i\}_{i=1}^q$, where $x_i \in \gX $ is the input and $y_i \in \{0,1\}^{K-1}$ is the one-hot label. Let $f$ and $h$ be the true underlying predictor layer and representation function. Then the output is $f \circ h(x) \in \sR^{K-1}$. Assume each label $\{y\}_i$ is generated from a conditional distribution of a multinomial logistic regression model: $y \sim \gP(\cdot|f\circ h(x))$,
\begin{align*}
	\gP(y |f \circ h(x)) = e^{y^\top f \circ h(x) - \Phi(f \circ h(x))}
\end{align*}
where $\Phi(x) = \log{( 1+\sum_{s=1}^{K-1} e^{x_s}}), x \in \sR^{K-1}$ and $y$ is an one-hot label.
\end{assumption}
\begin{remark}
It is straight forward to see that $ \gP(y |f \circ h(x))$ is normalized to 1.
\end{remark}
Intuitively, the assumption states that the data used for classification follow a multinomial logistic regression structure. 

\subsection{Task-Relatedness and Diversity}
We shall use the following definitions, which are natural analogies of those in~\citet{tripuraneni2020theory} for multi-task transfer learning. Being in the same framework of developing the diversity of the pre-training phase, ~\citet{tripuraneni2020theory} aimed at improving correlations between $K$ separate and easy tasks, while we show the diversity across various classes in a single but comprehensive pre-training task has prominent effects.

To measure the ``closeness'' between the learned representation and true underlying feature representation, we use the following metric, following~\citet{tripuraneni2020theory}
\begin{definition} 
\label{defn:pre_diff}
	Let $h \in \gH$ be the optimal representation function and $h' \in \gH$ be any representation function. 
	Let $f^{\pre} \in \gF^\pre$ be the optimal pre-training predictor on top of $h$.
	The pre-training representation difference is defined as:
	\begin{align*}
	& d_{\gF^{\pre}, f^{\pre}}(h^\prime; h)=\\
	&\inf_{f' \in \gF^{\pre}} \mathop{\E}\limits_{x^{\pre},y^{\pre}} \left[ \ell(f^\prime \circ h^\prime (x^{\pre}), y^{\pre}) - \ell(f^{\pre}\circ h(x^{\pre}), y^{\pre}) \right]
	\end{align*} 
	where the expectation is over the pre-training data distribution.
\end{definition}
Intuitively, this measures the performance difference between the optimal predictor and the best possible predictor given a representation $h'$.

For transfer learning, we also need to introduce a similar concept on the downstream task.
\begin{definition}
	\label{defn:down_diff}
	Let $h \in \gH$ be the optimal representation function and $h' \in \gH$ be any representation function.
For the downstream task, for a function class $\gF^{\down}$,  let $f^{\down} \in \gF^\down$ be the optimal pre-training predictor on top of a specific $h$.
We define the worst-case representation difference between $h$ and $h^\prime \in \gH$ as:
	\begin{align*}
	& d_{\gF^\down}(h^\prime; h)
	= \\
	&\sup_{f^\down \in \gF^\down} \inf_{f' \in \gF^{\down}} \mathop{\E}\limits_{x^\down, y^\down } \left[\ell(f^\prime \circ h^\prime (x^{\down}), y^{\down}) - \ell(f^{\down} \circ h(x^{\down}), y^{\down}) \right]
	\end{align*} 
where the expectation is over the data distribution of the downstream task. Here, the supremum is taken over $\{ f^d | f^\down \in \gF^\down, \text{$f^d$ is the optimal predictor on $h \in \gH$}\}$.
\end{definition}

We now introduce the key notion of \emph{diversity}, which measures how well a learned representation, say $h'$, from the pre-training task can be transferred to the downstream task.
%Since the learner does not have direct access to a signal from the representation, they can only observe partial information about the representation channeled through the composite functions $f^* \circ h^*$. If a particular direction/component in $h^*$
%is not seen by a corresponding task $f_1^*$
%in the training phase, that component of the representation $h^*$ cannot be distinguished from a corresponding one in a spurious $h^\prime$. Task diversity essentially encodes the
%ratio of Definition 1 and Definition 2 (i.e. how well the training tasks can cover the space captured by the representation $h^*$ needed to predict new tasks). 
\begin{definition}
	\label{defn:diverse}
	Let $h \in \gH$ be the optimal representation function.
		Let $f^{\pre} \in \gF^{\pre}$ be the optimal pre-training predictor on top of $h$.
The \textbf{diversity parameter} $\nu >0$ is the largest constant that satisfies 
%	For a function class F, we say function $f_1 \in \gF$ is $\nu$-\textbf{diverse} over $\gF^{\pre}}$ for a representation
%	$h$, if uniformly for all $h^\prime
%	\in \gH,$
	\begin{align}
d_{\gF^{\down}}(h^\prime; h)\le \frac{d_{\gF^{\pre}, f^{\pre}}(h^\prime; h)}{\nu}, \forall h' \in \gH.
	\end{align}
\end{definition}
% \jianshu{@Yulai: Shall we say instead: The diversity parameter is defined to be a constant $\nu>0$ that satisfies: $d_{\gF^{\down}}(h^\prime; h)\le \frac{d^{\pre}_{\gF^{\pre}, f^{\pre}}(h^\prime; h)}{\nu}, \qquad \forall h' \in \mathcal{H}$? To be more specific, shall we say ``a constant $\nu>0$'' or ``the largest $\nu>0$ that satisfies...''?}
The interpretation of $\nu$ is that it serves as a task-relatedness parameter.
While Definition~\ref{defn:pre_diff}-~\ref{defn:diverse} are naturally defined from inspecting the pre-training procedure, it is not trivial to use these definitions to derive statistical guarantees.
In particular, one of our key technical challenge is to show the least singular value of the last linear layer serves as a lower bound of the diversity parameter when $\gF^{\pre}$ and $\gF^{\down}$ are linear function classes.

\subsection{Model Complexities}
Lastly, we need to introduce some notions to measure the complexity of the function classes considered.
In this paper, we consider Gaussian complexity which quantifies the extent
to which the function in the class $\gQ$ can be correlated with a noise sequence of length $n \times r$. 

% \simon{What did you change here?}
% \yulai{
\begin{definition}[Gaussian Complexity] Let $\mu$ be a probability distribution on a set $\gX \subset \sR^d$ and suppose that $x_1, \cdots, x_n$ are independent samples selected according to $\mu$. Let $\gQ$ be a class of functions mapping from
$\gX$ to $\sR^r$. Define random variable
\begin{equation}
	\hat{G}_{n}(\gQ) = \mathop{\E}\limits_{{g_{ki} \sim \gN(0,1)}} \left[\sup_{q \in \gQ} \frac{1}{n} \sum_{k=1}^r \sum_{i=1}^n g_{ki}  q_k(x_i)\right]
\end{equation}
as the empirical Gaussian complexity, where $q_k(x_i)$ is the $k$-th coordinate of the vector-valued function $q(x_i)$, $g_{ki}$ $(k\in [r], i\in [n])$ are independent standard normal random variables. The Gaussian complexity of $\gQ$ is $G_n (\gQ) = E_{\mu} \hat{G}_{n}(\gQ)$.
\end{definition}
% }
% For a generic vector-valued function class $\gQ$ containing functions $q(\cdot): \sR^d \mapsto \sR^r$, and $N$ data points, $X= (x_1,\cdots, x_N)^\top$, the empirical Gaussian complexity is defined as
% \begin{equation}
% 	\hat{G}_{X}(\gQ) = \E_g \left[\sup_{q \in \gQ} \frac{1}{N} \sum_{k=1}^r \sum_{i=1}^N g_{ki}  q_k(x_i)\right], \qquad g_{ki} \sim \gN (0,1) 
% \end{equation}
% where $g = g_{ki, k\in[r], i\in[N]}$, and $q_k(\cdot)$ is the $k$-th coordinate of the vector-valued function $q(\cdot)$. We define the corresponding population Gaussian complexity as $G_N (\gQ) = E_{X} [\hat{G}_{X}(\gQ)]$, where the expectation is taken over the distribution of
% data samples $X$ . Intuitively, $G_N (\gQ)$ measures the complexity of $\gQ$ by the extent to which functions in the class $\gQ$ can correlate with random noise $g_{ki}$.
Our main results are stated in terms of the Gaussian complexity.
In Section~\ref{sec:subspace} and~\ref{sec:nn} we will plug in existing results of the Gaussian complexity of certain function classes to obtain concrete bounds.

% \yulai{
We will need the following worst-case Gaussian complexity for the pre-training predictor within $\gF^\pre$
\begin{gather}
\bar{G}_n(\gF^{\pre}) = \max_{h(x_1), \cdots, h(x_n)} \hat{G}_n(\gF^{\pre} | h  (x^\pre)),
\end{gather}
here $h \in \gH$ and $x^\pre = x_1, \cdots, x_n \in \gX^\pre$. Similarly we define $\bar{G}_m(\gF^{\down})$ as the worst-case Gaussian complexity for the downstream predictor within $\gF^\down$.
% where $Z^{\pre}$ and $Z^{\down}$ are the domains induced by any samples from $\gX^{\pre}, \gX^{\down}$ and any representation $h \in \gH$.
% }

We note that a closely related notion is Rademacher complexity. The empirical Rademacher complexity and Gaussian complexity only differ by a log factor~\citep{ledoux1991probability}. We use Gaussian complexity in this work for its benign properties brought by Gaussian distribution.
% \simon{@Yulai: cite.} \yulai{done}

% \simon{@Yulai do we need Rademacher?}
% Analogously to the above we can define the empirical and population Rademacher complexities for generic vector-valued
% functions as,
% \begin{equation}
% 	\hat{R}_{X}(\gQ) = \E_g [\sup_{q \in \gQ} \frac{1}{N} \sum_{k=1}^r \sum_{i=1}^N \epsilon_{ki}  q_k(x_i)], \qquad \epsilon_{ki} \sim \text{Rad}(\frac{1}{2})\text{ } i.i.d,
% \end{equation}
% its population counterpart is defined as $R_N (\gQ) = E_{X} [\hat{R}_{X}(\gQ)]$.

%% file: main_thm.tex
In this section, we present our main theoretical results.
In Section~\ref{sec:main_theorem} we present an analysis in terms of the diversity parameter for general loss function under certain regularity conditions.
In Section~\ref{sec:nlp_case}, we specialize the result to a setting that is most relevant to NLP pre-training applications, where $\gF^{\pre}$ and $\gF^{\down}$ are sets of linear functions and the loss is cross-entropy. In this particular case, our key result will show that one can use the singular value of the last layer to bound the diversity parameter.
In Section~\ref{sec:subspace} and~\ref{sec:nn} we instantiate our bounds on two concrete representation function classes: linear subspace and multi-layer network to showcase our main results.

\subsection{Main Theorem}
\label{sec:main_theorem}
In this subsection, we present our generic end-to-end transfer learning guarantee for multi-class transfer learning problems. 
We do not impose any specific function class formulations. Throughout this subsection, we only make the following mild regularity assumptions to make our results general.
% on the loss function $\ell(\cdot, \cdot)$, function class $\gF^{\pre}$, and  function class of shared representation  $\gH$.
% \simon{What's the difference between $\ell$ and $\ell'$?} \yulai{two losses for pre-training and downsteam learning.}
\begin{assumption}[Regularity Conditions]\label{assump_regularity}
We assume the following regularity conditions hold:
	\begin{itemize}
		\item In pre-training, $\ell(\cdot, \cdot)$ is $B^{\pre}$-bounded, and $\ell(\cdot, y)$ is $L^{\pre}$-Lipschitz for all $y$. 
		\item In downstream task, $\ell(\cdot, y)$ is $B^{\down}$-bounded and $L^{\down}$-Lipschitz for all $y$.
		\item Any predictor $f \in \gF^{\pre}$ is $L(\gF^{\pre})$-Lipschitz with respect to the Euclidean distance.
		\item Predictors are bounded: $\| f \circ h(x)\| \le D_{\gX^{\pre}}$ for any $x \in \gX^{\pre}, h \in \gH, f\in \gF^\pre$. Similarly $ \|f \circ h(x)\| \le D_{\gX^{\down}}$ for any $x \in \gX^{\down}, h \in \gH, f\in \gF^\down$.
	\end{itemize}
\end{assumption}
Specifically, one can show that common task-dependent losses satisfy these conditions. For example, when $\ell$ is the cross-entropy loss for $k-$class classification (cf. Section~\ref{sec:nlp_case}), we prove that $\ell$ is $\sqrt{k-1}-$Lipschitz and $D_\gX-$bounded where $\gX$ denotes the input data domain.

Under these assumptions, we have the following quantitative guarantee. 
\begin{theorem} \label{thm_main}
Under Assumption~\ref{assump_realizability} and~\ref{assump_regularity}, for a given fixed failure probability $\delta$, with probability at least $1-\delta$ we have the Transfer Learning Risk upper bounded by:
	\begin{align*}
		&O\Bigg{(} \frac{1}{\nu}\Bigg{\{} L^{\pre} \Bigg{[} \log{(n)} [L(\gF^{\pre}) G_{n}(\gH)+\bar{G}_n(\gF^{\pre})] + \frac{\sqrt{k} D_{\gX^{\pre}}}{n^2} \Bigg{]}\\
		&\qquad + B^{\pre} \sqrt{\frac{\log{(\nicefrac{1}{\delta})}}{n}}   \Bigg{\}}  + L^{\down} \bar{G}_m (\gF^{\down})   +  B^{\down} \sqrt{\frac{\log(\nicefrac{1}{\delta})}{m}}\Bigg{)}.
	\end{align*}
\end{theorem}

The first line comes from the pre-training ERM procedure and it accounts for the error of using an approximate optimal representation $\hat{h} \approx h$. The second line characterizes the statistical error of learning the downstream-task predictor $f^{\down}$ from $m$ samples.
Note the diversity parameter appears in the denominator, which relates the pre-training risk to the transfer learning risk.
Theorem~\ref{thm_main} shows the risk would be small if the Gaussian complexities are small. We expect that $G_{n}(\gH) \gg \bar{G}_m(\gF^\down)$ since $\gH$ is often expressive representation functions, while $\gF^\down$ is linear classifiers generally.
We will show concrete examples where $G_{n}(\gH)$ and $\bar{G}_n(\gF^{\pre})$ are $O(\sqrt{1/n})$ and  $\bar{G}_m$ scales as $O(\sqrt{1/m})$.
We believe this theorem applies broadly beyond the concrete settings considered in this paper.

In comparison with previous results, transfer learning risk analyses in ~\citep{du2020few, tripuraneni2020theory} focus on scalar output. 
Their results cannot be applied to multi-class transfer learning tasks. 
In Theorem~\ref{thm_main}, we generalize the analysis in~\citep{tripuraneni2020theory} to handle multi-class classification where the output is high dimensional (number of classes).
Technically, in the proof, we introduce a vector-form Rademacher complexity chain rule for disassembling composite function classes by making use of the vector-form Rademacher contraction property~\citep{maurer2016vector}.

\subsection{Multi-class Classification with Cross-entropy Loss}
\label{sec:nlp_case}
Now we specialize the general results to the setting that is of most interest to NLP pre-training,
where the loss function $\ell$ is cross-entropy and the $\gF^{\pre}$ and $\gF^{\down}$ are sets of linear functions. 
This choice is consistent with the NLP pre-training: e.g., BERT~\citep{devlin2018bert} uses transformers as the representation learning function class $\gH$ and uses word-embedding matrices as $\gF^{\pre}$.
% We further improve Theorem~\ref{thm_main} to be problem-specific. 

Formally we define
\begin{align*}
	\gF^{\pre} &= \{ f| f(z) = \alpha^\top z, \alpha \in \sR^{r\times (k-1)},\\
	&\qquad \qquad\|\alpha_s\| \le c_1\text{ for all } s \in [k-1], \|\alpha^\top z\| \le c_2 \}\\
	\gF^{\down} &= \{ f| f(z) = \alpha^\top z, \alpha \in \sR^{r\times (k^\prime-1)},\\
	&\qquad \qquad\|\alpha_s\| \le c_0 \text{ for all } s \in [k^\prime-1], \|\alpha^\top z\| \le c_3 \}
\end{align*}
where $c_0, c_1, c_2$ and $c_3$ are some positive constants.
Then the regularity conditions are instantiated as:
\begin{itemize}
	\item Pre-training loss $\ell(\cdot, y)$ is $\sqrt{k - 1}$-Lipschitz and $B^{\pre} = D_{\gX^{\pre}}$-bounded.
	\item Downstream loss is $\sqrt{k^\prime - 1}$-Lipschitz and $B^{\down} = D_{\gX^{\down}}$-bounded.
	\item Any $f \in \gF^{\pre}$ is $L(\gF^{\pre}) = c_1 \sqrt{k-1}$-Lipschitz w.r.t. the $\ell_2$ distance.
\end{itemize}
% See Appendix~\ref{pf_thm_cross_entropy} for detailed derivations.

% We note that since $z$ is the output of  representation layer $\gH \subset \sR^d \mapsto \sR^r$. Let $x$ be the original input data. It is commonly assumed that $\|x\| = \gO(1)$. In this case, representation layer generally has . Note that . Therefore, when , we have . 

Next, we discuss our main assumption that relates the diversity parameter to a concrete quantity of the last linear layer.
%We make the following assumption on the true training classes being diverse and both the training and new task being
%normalized.
\begin{assumption}[Lower Bounded Least Eigenvalue] \label{assump_least_singular}
	Let the optimal linear predictor at the last layer for pre-training be $\alpha^{\pre} \in \mathbb{R}^{r \times (k-1)}$, $\tilde{\nu} \triangleq \sigma_r(\alpha^{\pre} \left(\alpha^{\pre}\right)^\top)  > 0$ where $\sigma_r$ is the $r$-biggest eigenvalue.
\end{assumption}
Similar assumptions have been used in multi-task representation learning~\citep{du2020few,tripuraneni2020provable,tripuraneni2020theory}, and are shown to be necessary~\citep{maurer2016benefit,du2020few}.
Different from their versions, our assumption is tailored for the multi-class classification setting. 
We provide proof sketches on how $\tilde{\nu}$ serves as a lower bound for the diversity parameter $\nu$ (cf. Lemma~\ref{lem_lowerbound_nu}) in Section~\ref{subsec:diversity}, where we introduce new techniques for analysis.

Intuitively, this assumption ensures that the pre-training task matrix spans the entire $r$-dimensional space and thus covers the
output of the optimal representation $h(\cdot) \in \sR^r$. 
This is quantitatively captured by the  $\sigma_r(\alpha^{\pre} \left(\alpha^{\pre}\right)^\top)$, which measures how spread out these vectors are in $\sR^r$.

%The training task will be well-conditioned in the sense that
%$\sigma_1(\alpha_1 \alpha_1^\top)/\sigma_r(\alpha_1 \alpha_1^\top) \ge \gO(1)$ (w.h.p.) for example, if each $\alpha_1 \sim \gN (0, \frac{1}{\sqrt{r}} \Sigma)$, where $\sigma_1(\Sigma)/\sigma_r(\Sigma) \ge \gO(1)$.
We now state our theorem for this specific setting.
\begin{theorem} \label{thm_cross_entropy}
Under Assumption~\ref{assump_realizability}, ~\ref{assump_samples}, ~\ref{assump_least_singular}, with probability at least $1-\delta$ we have the Transfer Learning Risk upper bounded by:
	\begin{align*}
		& O \Bigg{(} \frac{1}{\tilde{\nu}}\Bigg{\{}  \sqrt{k} \Bigg{[} \log{(n)} [\sqrt{k} G_{n}(\gH)+\bar{G}_n(\gF^{\pre})] + \frac{\sqrt{k} D_{\gX^{\pre}}}{n^2} \Bigg{]} \\
		& +D_{\gX^{\pre}}\sqrt{\frac{\log{(\nicefrac{1}{\delta})}}{n}}   \Bigg{\}} + \sqrt{k^\prime} \mathop{\E}\limits_{\gX^\down} \hat{G}_{m} (\gF^{\down} | \hat{h} \circ x^{\down})\\
		&+ \sigma \sqrt{\frac{\log(\nicefrac{1}{\delta})}{m}} + D_{\gX^{\down}} \sqrt{\frac{\log(\nicefrac{1}{\delta})}{m}} \Bigg{)}
	\end{align*}
Here $\E_{\gX^\down} \hat{G}_{m} (\gF^{\down} | \hat{h} \circ x^{\down} )$ is Gaussian complexity of embeddings
$$
	\hat{h} \circ x^\down = \{ \hat{h}(x_1), \cdots, \hat{h}(x_m) | x^\down= x_1, \cdots, x_m \in \gX^{\down} \}
$$
where the expectation is over $\gX^\down$, and
 $\sigma^2 = \frac{1}{m} \sup_{f \in \gF^{\down}} \sum_{i=1}^m Var(\ell(f \circ \hat{h}(x_{i}^{\down}), y_{i}^{\down}))$ is the  maximal variance over $\gF^\down$.
\end{theorem}
We remark that in Theorem~\ref{thm_cross_entropy}, since we specialize to the case where $\gF^{\pre}$ and $\gF^{\down}$ are sets of linear functions, we can replace the term $L^{\down} \cdot \bar{G}_m (\gF^{\down})$ in Theorem~\ref{thm_main} by $ (\sqrt{k^\prime} \cdot \E_{\gX^{\down}} \hat{G}_m (\gF^{\down} | \hat{h} \circ x^{\down})+ \sigma \sqrt{\nicefrac{\log(\nicefrac{1}{\delta})}{m}})$ by utilizing the functional Bernstein inequality.
This improvement can help us obtain Theorem~\ref{thm1}.
See Appendix~\ref{pf_cross_entropy} for details.

Now we discuss the interpretation of Theorem~\ref{thm_cross_entropy}.
Typically, $\bar{G}_n(\gF^{\pre})$ is much smaller than $G_{n}(\gH)$ because $G_{n}(\gH)$ represents the complexity of the representation function, which is often complex.
In practice, $n$ is often large.
Therefore, in the benign case where $\tilde{\nu} = \Theta\left(k\right)$ (when the condition number of $\alpha^{\pre}$ is $O(1)$), the dominating term will be $G_{n}(\gH)$.
As we will show in the following subsections, this term typically scales as $O(\sqrt{\nicefrac{1}{n}})$.
Together, the theorem clearly shows when 1) the number of pre-training data is large, and 2) the least singular value of the last linear layer for pre-training is large, the transfer learning risk is small.
On the other hand, if $\tilde{\nu}$ is small, then the bound becomes loose.
This is consistent with prior counterexamples on multi-task pre-training~\citep{maurer2016benefit,du2020few} where the diversity is shown to be necessary.

\subsection{What is \emph{diversity parameter} for Linear Layers?}
\label{subsec:diversity}
To prove Theorem~\ref{thm_cross_entropy}, one of our key technical contributions is to show the following lemma that bridges the gap between Theorem~\ref{thm_main} and Theorem~\ref{thm_cross_entropy} 
\begin{lemma} 
\label{lem_lowerbound_nu}
	 Under Assumption~\ref{assump_realizability}, ~\ref{assump_samples}, ~\ref{assump_least_singular}, we have
	\begin{equation}
		d_{\gF^\down}(\hat{h}; h) \le \frac{1}{\Omega(\tilde{\nu})} d_{\gF^\pre, f^\pre}(\hat{h}; h).
	\end{equation}
\end{lemma}
In intuition, it says that $\tilde{\nu}$ could serve as a lower bound for the diversity parameter $\nu$. The take-away message is that we may wish to achieve a higher $\tilde{\nu}$ in order to increase the diversity of the pre-training models, thus improving its generality to various downstream tasks.

Technically, in proving the results we shall need to apply a \emph{modified self-concordance} condition for better characterizing multinomial logistic regression~\citep{bach2010self}.
We note that the proofs of this part is very different from the multi-task setting studied in previous works~\citep{du2020few,tripuraneni2020theory}.

We define some additional notations for clarity and simplicity in this subsection. Let $\alpha^\prime$ and $\alpha$ denote the parameters for $\hat{f}^\pre$ and $f^\pre$ respectively. Let $\Phi(x) = \log{( 1+\sum_{s=1}^{k-1} e^{x_s}})$, for $x \in \sR^{k-1}$, which is widely seen in multinomial regression tasks because the cross-entropy loss is inherently analogous to multinomial logistic loss. 

In this subsection, we emphasize on the techniques required to show the following lemma, which incorporates the main difficulty in the proof for Lemma~\ref{lem_lowerbound_nu}.

\begin{lemma} \label{lem_quadratic_bound} 
 The Kullback-Leibler (KL) divergence between the true underlying conditional distribution of a multinomial logistic model and the distribution we obtained can be bounded from both sides with quadratic loss,
	\begin{align*}
		  &\quad c_0 e^{ - 10 q_0 } \left\| {\alpha^\prime}^\top \hat{h}(x^\pre) - \alpha^\top h(x^\pre) \right\|^2\\
		  &\le KL \left[\gP(\cdot| \alpha^\top h(x^\pre)), \gP(\cdot| {\alpha^\prime}^\top \hat{h}(x^\pre)) \right]\\
		 &\le \frac{1}{2} \left\| {\alpha^\prime}^\top \hat{h}(x^\pre) - \alpha^\top h(x^\pre) \right\|^2,
	\end{align*}
	where $c_0 = \frac{1}{2} \lambda_{min}(\Phi''(\alpha^\top h(x^\pre)))$ is the least eigenvalue of Hessian matrix for $\Phi$, $q_0 = \max(\|{\alpha^\prime}^\top \hat{h}(x^\pre)\|, \| \alpha^\top h(x^\pre) \|)$.
	\end{lemma}
\begin{remark}
For the left hand side, the expression is related to the least eigenvalue of the Hessian matrix at $\alpha^\top h(x^\pre)$, whereas the least eigenvalue would depend on an unknown intermediate-term $x^\prime$ if we adopt Taylor's expansion.
\end{remark}

\textit{Proof of Lemma~\ref{lem_quadratic_bound}.}
Below we use $x$ for $x^\pre$ for clarity. For generalized linear models,
\begin{align*}
    KL &\left[\gP(\cdot| \alpha^\top h(x)), \gP(\cdot| {\alpha^\prime}^\top \hat{h}(x)) \right]=\Phi({\alpha^\prime}^\top \hat{h}(x)) -\\
    &  \Phi(\alpha^\top h(x))- \nabla \Phi(\alpha^\top h(x))^\top ({\alpha^\prime}^\top \hat{h}(x) - \alpha^\top h(x)).
\end{align*}
Hence the divergence serves as the second-order remainder term according to Taylor's theorem.

% Taylor's expansion shows that
% \begin{equation*}
% 	\Phi({\alpha^\prime}^\top \hat{h}(x)) = \Phi(\alpha^\top h(x)) + \nabla \Phi(\alpha^\top h(x))^\top ({\alpha^\prime}^\top \hat{h}(x) - \alpha^\top h(x)) + o \left(\| {\alpha^\prime}^\top \hat{h}(x) - \alpha^\top h(x) \|^2 \right).
% \end{equation*}
For the right hand side, the gradient of $\Phi(x)$ at $i^{th}$-coordinate $\frac{\partial \Phi}{\partial x_i} = \nicefrac{e^{x_i}}{1 + \sum_s e^{x_s}}$, the Hessian matrix is
\begin{equation*}
	\frac{\partial^2 \Phi}{\partial x_i \partial x_j}=
	\left\{
	\begin{array}{lr}
		\frac{e^{x_i} \cdot ( 1 + \sum_{s \neq i} e^{x_s})}{(1 + \sum_s e^{x_s})^2}, & i = j \\
		&\\
		\frac{- e^{x_i} e^{x_j}}{(1 + \sum_s e^{x_s})^2}, & i \neq j.
	\end{array}
	\right.
\end{equation*}
Let $\sigma(x) = \frac{1}{1 +\sum_s e^{x_s}} \left[ e^{x_1}, \cdots, e^{x_{k-1}}\right]^\top$, the Hessian matrix can be restated as 
\begin{equation*}
	\nabla^2 \Phi = diag(\sigma(x)) - \sigma(x) \sigma(x)^\top.
\end{equation*} 
For any non-zero vector $y$, we have
\begin{align*}
	y^\top \nabla^2 \Phi y &= \sum_i \sigma(x)_i y_i^2 - \left( \sigma(x)^\top y \right)^2\\ &\le \max(\sigma(x)_i) \|y\|^2\\
	&\le \| y\|^2
\end{align*}
which implies its largest eigenvalue is no bigger than 1.

For the left hand side, it is very straightforward to see that the Hessian matrix is positive semi-definite. Though being non-negative, we point out that bounding the second-order remainder terms from below with quadratic loss requires new techniques which we discuss below.

Since multinomial logistic regression is not strongly-convex, we need to find new techniques that would present benign properties to characterize the local landscape.
Below we introduce a class of convex functions called \emph{modified self-concordant} functions, which would be useful in quantitative analysis.
\begin{definition}[Modified Self-concordance]
Suppose	$F$: $\sR^p \mapsto \sR$ is a three times differentiable convex function such that for some $R > 0$, for all $u, v\in \sR^p$, the function $g: t \mapsto F(u+tv)$ satisfies for all $t \in \sR$
	\begin{equation}
		|g'''(t)| \le R\|v\|_2 \times g''(t)
	\end{equation} 
\end{definition}
\paragraph{Properties of self-concordance} Self-concordance gives nice characterizations of local curvature of convex functions which plays important role in describing local convexity~\citep{bach2014adaptivity}. Some useful results are given upon this condition (see~\citep[Proposition 1]{bach2010self}), we list out the three main inequalities as below:
For all $w, v \in \sR^p, t \in \sR$,
\begin{align*}
	F(w+v) &\ge F(w) + v F'(w) +\\
	&\qquad \frac{v^\top F''(w)v}{R^2 \|v\|_2^2}\cdot \left( e^{-R\|v\|_2} + R\|v\|_2 -1 \right),\\
	F(w+v) &\le F(w) + v F'(w)+ \\
	&\qquad \frac{v^\top F''(w)v}{R^2 \|v\|_2^2} \cdot \left( e^{R\|v\|_2} - R\|v\|_2 -1 \right),\\
	e^{-t R \|v\|_2} &F''(w) \preceq F''(w+t v) \preceq e^{t R \|v\|_2} F''(w).
\end{align*}
The first two inequalities are refined characterizations of Taylor's expansion, while the last line presents bounds for Hessian matrix in the sense of positive semi-definiteness.

We find that multinomial logistic loss satisfies the \emph{modified self-concordance} condition with $R=5$.
\begin{proposition}\label{prop_self}
For all $u, v \in \sR^{k-1}$, the function $g: t \mapsto \Phi(u+tv)$ satisfies
	\begin{equation*}
		|g'''(t)| \le 5\|v\|_2 g''(t).
	\end{equation*}
\end{proposition}
See Appendix~\ref{pf_cross_entropy} for detailed derivations. Equipped with self-concordance, we are ready to give a lower bound of the divergence,
\begin{align*}
	& \quad \Phi({\alpha^\prime}^\top \hat{h}(x)) - \Phi(\alpha^\top h(x)) - \nabla \Phi(\alpha^\top h(x))^\top v \\
	&\ge \frac{1}{2} v^\top e^{- 5 \|v\|_2} F''(\alpha^\top h(x)) v\\
	&\ge \frac{1}{2} \lambda_{min}(\Phi''(\alpha^\top h(x))) \|v\|^2 e^{- 5 \|v\|_2}\\
	&\ge \frac{1}{2} \lambda_{min}(\Phi''(\alpha^\top h(x))) \|v\|^2 e^{- 5(\|{\alpha^\prime}^\top \hat{h}(x)\| + \|\alpha^\top h(x)\| )}\\
	&\ge \frac{1}{2} \lambda_{min}(\Phi''(\alpha^\top h(x))) \|v\|^2 \exp( -10 q_0 )
\end{align*}
where $v = {\alpha^\prime}^\top \hat{h}(x) - \alpha^\top h(x)$. This completes our proofs for Lemma~\ref{lem_quadratic_bound}. Please find the remaining details for completing the proof of
Lemma~\ref{lem_lowerbound_nu} in Appendix~\ref{pf_cross_entropy}.

%% file: setting1.tex
\subsection{Linear Subspace Representation}
\label{sec:subspace}
 Based on cross-entropy loss and linear predictors introduced in Section~\ref{sec:nlp_case}, we further assume the underlying representation is a projection onto a low-dimensional subspace. For $r \ll d$, let the representation be
\begin{align*}
	\gH &= \{ h | h(x) =B^\top x, B \in \sR^{d \times r}\},
\end{align*}
where $B$ is a matrix with orthonormal columns.
We require some additional regularity conditions. Following prior work~\citep{du2020few,tripuraneni2020theory},
we assume that $\|x\| \le D$ and input data distribution satisfies the following condition.
\begin{definition} \label{defn_regularity_conditions}
	The covariate distribution $P_x(\cdot)$ is $\Sigma$-sub-gaussian if for all $v \in \sR^d$,
	$$\E[\exp(v^\top x)] \le \exp \left(\frac{\|\Sigma^{\nicefrac{1}{2}}v\|^2}{2}\right)$$
	where the covariance $\Sigma$ further satisfies $\sigma_{max}(\Sigma) \le C$ and $\sigma_{min}(\Sigma) \ge c > 0$ for universal constants $c, C$.
\end{definition}

We have the following theorem that guarantees the performance of transfer learning.
\begin{theorem} \label{thm1}
	Suppose Assumption~\ref{assump_realizability}, ~\ref{assump_samples}, and~\ref{assump_least_singular} hold, data generation follows Definition~\ref{defn_regularity_conditions}. For a sufficiently large constant $c_4$, we assume $n \ge c_4 d, m \ge c_4 r$, $D \le c_4(\min(\sqrt{dr^2}, \sqrt{rm}))$.
	Then with probability at least $1-\delta$, we have the Transfer Learning Risk upper bounded by:
% 	\simon{check format.}\yulai{done.}
	\begin{align*}
		 & O \Bigg{(} \frac{1}{\tilde{\nu}} \Bigg{[}\sqrt{k} \log{(n)} \left( \sqrt{\frac{k d r^2}{n}} + k \sqrt{\frac{r}{n}} \right) +  \frac{k}{n^2} +  \sqrt{\frac{\log{(\nicefrac{1}{\delta})}}{n}}  \Bigg{]}\\
		 &+  (k^\prime)^{\frac{3}{2}} \sqrt{\frac{r}{m}} +  k^\prime \sqrt{ \frac{\log{(\nicefrac{1}{\delta})}}{m} }\Bigg{)}
	\end{align*}
\end{theorem}
%As we have noted, generally there has $k^\prime \ll k$ thus making efficient downstream task learning possible. 
To interpret this bound, consider the practically relevant scenario where $k' = O(1)$ (e.g., sentiment analysis), $m \ll n$, $k \ll n$ and $r \ll d$, in the benign case $\tilde{\nu} = \Omega\left(k\right)$, we have the transfer learning risk $\widetilde{O}\left(\sqrt{\nicefrac{dr^2}{n}} + \sqrt{\nicefrac{r}{m}}\right)$.
Note that this is exactly the desired theoretical guarantee because the first term accounts for using all pre-training data to learn the representation function and the second term accounts for using the downstream data to learn the last linear layer.
This is significantly better than not using pre-training, in which case the risk scales $O\left(\sqrt{\nicefrac{d}{m}}\right)$. Furthermore, for the linear representation learning setting, classic minimax bounds present a standard $\Omega(\sqrt{\nicefrac{d}{m}})$ lower rate, which is also worse than our upper bound with representation learning~\citep{foster2018logistic,abramovich2018high, barnes2019minimax}.

% The proof of Theorem~\ref{thm1} consists of two parts. 1) First we instantiate complexity quantities by substituting specific estimates into Theorem~\ref{thm_main}, then seek to lower bound the task-diversity parameter with task-dependent statistics. 2)

%% file: setting2.tex
\subsection{Deep Neural Network Representation}
\label{sec:nn}
In this subsection, we assume the underlying representation function to be a $\sigma = tanh$-activated neural network, which is often used in practice. Predictors are still required to be linear functions at the interest of NLP pre-training, i.e.,
\begin{align*}
    \gH &= \{  h| h(x) = W_K \sigma \left(W_{K-1} \sigma\left(\cdots \sigma \left(W_1 x \right) \right) \right)\},\\
	\gF^{\pre} &= \{ f| f(z) = \alpha^\top z, \alpha \in \sR^{r\times (k-1)}, \\
	&\qquad \|\alpha_s\| \le c_1 M(K)^2, s \in [k-1]. \|\alpha^\top z\| \le c_2 \},\\
	\gF^{\down} &= \{ f| f(z) = \alpha^\top z, \alpha \in \sR^{r\times (k^\prime-1)},\\ &\qquad \|\alpha_s\| \le c_0 M(K)^2,  s \in [k^\prime -1]. \|\alpha^\top z\| \le c_3 \}.
\end{align*}
Here $M$ refer to constants that only depend on the network configuration, which satisfy: 1) for each $p \in [K], \|W_p\|_{1, \infty} \leq M(p)$, and 2) $ \|W_K\|_{\infty \to 2} \le M(K).$

Adapt Gaussian complexity results in~\citet{golowich2017size} we have
\begin{gather*}
	G_{n}(\gH) \leq \widetilde{O} \left(\frac{r M(K)^3 \cdot D \sqrt{K} \cdot \Pi_{p=1}^{K-1} M(p)}{\sqrt{n}} \right),\\	
	 \bar{G}_n (\gF^{\pre} | h \circ x^\pre) \leq O \left( \frac{(k-1) M(K)^3}{\sqrt{n}} \right).
\end{gather*}
Now we are ready to state our theorem for this practical setting of NLP pre-training.
\begin{theorem} \label{thm2}
	Under Assumption~\ref{assump_realizability}, ~\ref{assump_samples}, and~\ref{assump_least_singular}, assume $M(K) \geq c_5$ for a universal constant $c_5$. Then with probability at least $1-\delta$, Transfer Learning Risk is upper bounded by
% 	\footnote{We ignore logarithmic terms for simplicity.}
	\begin{align*}
		&\widetilde{O} \Bigg{(} \frac{k r M(K)^3 \cdot D \sqrt{K} \cdot \Pi_{p=1}^{K-1} M(p)}{\tilde{\nu} \sqrt{n}}  + \frac{k^{\frac{3}{2}} M(K)^3}{\tilde{\nu} \sqrt{n}}\\
		&+  \frac{{k^\prime }^{\frac{3}{2}}M(K)^3}{\sqrt{m}} \Bigg{)}.
	\end{align*}
\end{theorem}
% Similar to Theorem~\ref{thm1}, this theorem shows for a sufficiently large number of pre-training data and a diverse last linear layer, the transfer learning risk is small. See Appendix~\ref{pf_app_2} for proof.
% \simon{@yulai, add some discussions, maybe check how \cite{tripuraneni2020theory} discuss this result.}
% \yulai{
To interpret this bound, one can easily show that a standard supervised learning paradigm without pre-training would have a sample complexity of $\widetilde{O}(k r M(K)^3 \cdot D\sqrt{K} \cdot \Pi_{p=1}^{K-1} M(p) /  \sqrt{m})$. Again, this theorem demonstrates: when 1) $n \gg m$ and 2) $\tilde{\nu}$ is large, the rate of transfer learning risk can be much smaller than that of the standard supervised learning algorithm.
% }

%% file: conclusion.tex
This work theoretically prove the benefit of multi-class pre-training using the notion of class diversity.
Our proof uses the vector-form Rademacher complexity chain rule and a modified self-concordance condition.

\paragraph{Future work} First, our work is based on realizability assumptions (cf. Assumption~\ref{assump_realizability} and ~\ref{assump_samples}) that are commonly adopted in transfer learning and classical PAC learning framework in order to present non-trivial statistical guarantees~\citep{maurer2016benefit,du2020few,tripuraneni2020theory,shalev2014understanding}. We believe our theorems can be extended to agnostic versions by relaxing these assumptions.

 Second, if the target task is well-aligned with the source tasks, one can define more fine-grained notions to capture the task relevance. An example is~\citep{chen2022active}, in which regression setting is studied. One interesting direction is extending their task relevance definition to the classification setting.
 
 Finally, there has been some interesting recent work showing that one can do pre-training (i.e., masked word prediction) with the downstream dataset itself (which is usually smaller than typical pre-training corpora) and get good results~\citep{krishna2022downstream}. Compared to the setting studied in this work, it might be harder to justify its performance through a ``diversity'' perspective because their settings are generally beyond the standard transfer learning scheme. Nevertheless, our interpretation of $\nu$ as a task-relatedness parameter might help shed light on these results, which is worthy of investigation.

%% file: proof.tex
In Section~\ref{sec:preliminary}, we have introduced Gaussian complexity. Let us restate for clarity.

Let $\mu$ be a probability distribution on a set $\gX \subset \sR^d$ and suppose that $x_1, \cdots, x_n$ are independent samples selected according to $\mu$. Let $\gQ$ be a class of functions mapping from
$\gX$ to $\sR^r$. Define random variable
\begin{equation}
	\hat{G}_{n}(\gQ) = \E \left[\sup_{q \in \gQ} \frac{1}{n} \sum_{k=1}^r \sum_{i=1}^n g_{ki}  q_k(x_i)\right] \qquad
\end{equation}
as the empirical Rademacher complexity, where $q_k(x_i)$ is the $k$-th coordinate of the vector-valued function $q(x_i)$, $g_{ki}$ $(k\in [r], i\in [n])$ are independent Gaussian $\gN(0, 1)$ random variables.  The Gaussian complexity of $\gQ$ is $G_n (\gQ) = E_{\mu} \hat{G}_{n}(\gQ)$.

Analogously to the above we can define the empirical Rademacher complexity for vector-valued functions as
\begin{equation}
	\hat{R}_{n}(\gQ) = \E \left[\sup_{q \in \gQ} \frac{1}{N} \sum_{k=1}^r \sum_{i=1}^N \epsilon_{ki}  q_k(x_i) \right]
\end{equation}
where $\epsilon_{ki} (k \in [r], i \in [n])$ are independent Rademacher  $\text{Rad}(\frac{1}{2})$ random variables. Its population counterpart is defined as $R_n (\gQ) = E_{\mu} [\hat{R}_{n}(\gQ)]$. Note that the superscripts existing in $\hat{G}$ and $\hat{R}$ imply that they are empirical measures.

\subsection{Proofs for Section~\ref{sec:main_theorem}} \label{pf_thm_main}
We illustrate Theorem~\ref{thm_main} in two stages. First we show pre-training representation difference can be upper bounded by constants and function class complexities. Then we transfer it to the downstream task through the diversity parameter.

\paragraph{Pre-training}
\begin{theorem} \label{thm_train_phase}
	In pre-training, with probability at least $1-\delta$, it holds that:
	\begin{align*}
		& \quad d_{\gF^{\pre}, f^{\pre}}(h^\prime; h)\\
		&\le 4\sqrt{\pi}L^{\pre} G_n(\gF^{\pre} \circ \gH) + 4B^{\pre} \sqrt{\frac{\log(\nicefrac{2}{\delta})}{n}}\\
		&\le 4096L^{\pre} \left[\frac{\sqrt{k-1}D_{\gX^{\pre}}}{n^2} + \log(n) [L(\gF^{\pre}) G_n(\gH) + \bar{G}_n(\gF^{\pre})]\right] + 4B^{\pre} \sqrt{\frac{\log(\nicefrac{2}{\delta})}{n}}.
	\end{align*}
\end{theorem}

\begin{proof}
	We begin with
	\begin{align*}
		d_{\gF^{\pre}, f^{\pre}}(h^\prime; h) \le 2 \sup_{f \in \gF^{\pre}, h \in \gH} |R_{\pre}(f^\pre,h) - \hat{R}_{\pre}(f^\pre,h)|.
	\end{align*}
From the definition of Rademacher complexity~\citep[Theorem 4.12]{wainwright2019high}, with probability at least $1-2\delta$ we have
	\begin{equation*}
		\sup_{f^\pre \in \gF^{\pre}, h \in \gH} |R_{\pre}(f^\pre,h) - \hat{R}_{\pre}(f^\pre,h)| \le 2 R_{n} (\ell(\gF^\pre \circ \gH)) + 2B^\pre \sqrt{\frac{\log(\nicefrac{1}{\delta})}{n}}.
	\end{equation*}
	
% 	One often use contraction lemma for composite Rademacher complexity, 
	Next, we apply the vector contraction inequality~\citep{maurer2016vector}. For function class $\gF$ whose output is in $\sR^K$ with component $f_k(\cdot)$, and the function $(h_i)$s are some $L$-Lipschitz functions: $\sR^K \mapsto \sR$, we have
	\begin{align}
		\E_\epsilon \sup_{f \in \gF} \sum_{i=1}^n \epsilon_i h_i(f(x_i)) \le \sqrt{2} L \E_\epsilon \sup_{f\in \gF} \sum_{i=1}^n \sum_{k=1}^K \epsilon_{ik} f_k(x_i).
	\end{align} 
	
	Hence for loss function $\ell$ satisfying $|\ell(x) - \ell(y)| \le L^\pre \|x-y\|_2, \forall x,y \in \sR^{k-1}$, the $f$ takes value in $\sR^{k-1}$ with component functions $f_s(\cdot), s\in [k-1]$, we have that population Rademacher complexity can be bounded by
	\begin{align*}
		R_n(l(\gF^{\pre} \circ \gH)) &= \E_{X^\pre} \frac{1}{n} \E_\epsilon \sup_{f \in \gF^{\pre}, h \in \gH} \sum_{i=1}^n \epsilon_i \ell(f \circ h(x_i^\pre))\\
		&\le \E_{X^\pre} \frac{1}{n} \sqrt{2}L^\pre  \E_\epsilon \sup_{f \in \gF^\pre, h\in \gF} \sum_{i=1}^n \sum_{s=1}^{k-1} \epsilon_{is} f_s (h (x_i^\pre))\\
		&= \sqrt{2}L^\pre R_n(\gF^\pre \circ \gH)\\
		&\le \sqrt{\pi} L^\pre G_n(\gF^\pre \circ \gH)
	\end{align*}
	where the last line uses the fact that Rademacher complexity is upper bounded by Gaussian complexity: $R_n (\gF^\pre \circ \gH) \le \sqrt{\frac{\pi}{2}} G_n(\gF^\pre \circ \gH)$. Therefore we have, with probability at least $1-\delta$,
	\begin{align*}
		& \quad d_{\gF^{\pre}, f^{\pre}}(h^\prime; h) \\
		&\le 4\sqrt{\pi}L^\pre G_n(\gF^\pre \circ \gH) + 4B^\pre \sqrt{\frac{\log(\nicefrac{2}{\delta})}{n}}\\
		&\le 4096L^\pre \left[\frac{\sqrt{k-1}D_{\gX^\pre}}{n^2} + \log(n) [L(\gF^\pre) G_n(\gH) + \bar{G}_n(\gF^\pre)]\right] + 4B^\pre \sqrt{\frac{\log(\nicefrac{2}{\delta})}{n}},
	\end{align*}
	where the last line uses
	decomposition of $G_n(\gF^\pre \circ \gH)$ into the individual Gaussian complexities of $\gH$ and $\gF^\pre$, leverages an expectation version of novel chain rule for
	Gaussian complexities (Lemma~\ref{lem_decomposition}).
\end{proof}
In the spirit of Gaussian complexity decomposition theorem~\citep[Theorem 7]{tripuraneni2020theory}, we introduce the following decomposition result upon vector-form Gaussian complexities.
\begin{lemma}\label{lem_decomposition}
We have the following vector form Gaussian complexity decomposition:
	\begin{align}
		\hat{G}_n (\gF^\pre \circ \gH) &\le  \frac{8 \sqrt{k-1} D_{\gX^\pre}}{n^2} + 512 C(\gF^\pre \circ \gH) \cdot \log{(n)} 
	\end{align}
\end{lemma}
where we use $C(\gF^\pre \circ \gH) =  L(\gF^\pre) \cdot \hat{G}_{n} (\gH) + \bar{G}_n (\gF^\pre)$ to represent the complexity measure of the composite function class.
\begin{proof}
% \simon{@Yulai: what is the main difference from \cite{tripuraneni2020theory}?}

% \yulai{
Our proof extends \citep[Theorem 7]{tripuraneni2020theory}, which focuses on a multi-task scalar formulation. We further extend it to multi-class vector formulation. 
Specifically, on top of the representation class $\gH$, they need to handle $\gF^{\otimes t}$ ($t$ is the number of tasks) while our objective is a single function class $\gF^\pre$ of higher dimension ($\gF^\pre$ is $(k-1)$-dimensional for a $k$-class classification task). We note that our proof technique and that of previous works~\citep{tripuraneni2020theory, maurer2016benefit} both hinge on several properties of Gaussian processes.
% }

	To bound the empirical composite function class $\gF^\pre (\gH)$, note that vector-form Gaussian complexity is defined as
	\begin{align*}
		\hat{G}_{n} (\gF^\pre \circ \gH ) &= \E \left[\frac{1}{n} \dot \sup_{f(h) \in \gF^\pre (\gH)} \sum_{s=1}^{k-1} \sum_{i=1}^n g_{is} f_s(h(x_i^\pre))\right]\\
		&= \frac{1}{\sqrt{n}} \E[ \sup_{f(h) \in \gF^\pre (\gH)} Z_{f(h)}]
	\end{align*}
	where we define mean-zero process $Z_{f(h)} = \frac{1}{\sqrt{n}} \sum_{s=1}^{k-1} \sum_{i=1}^n g_{is} f_s(h(x_i^\pre))$, then $\E \sup_{f(h)} Z_{f(h)} = \E \sup_{f(h)} Z_{f(h)} - Z_{f^\prime(h^\prime)} \le \E \sup_{f(h), f^\prime(h^\prime)} Z_{f(h)} - Z_{f^\prime(h^\prime)}$. We further notice that $Z_{f(h)} - Z_{f^\prime(h^\prime)}$ is a sub-gaussian random variable parameter
	\begin{align*}
		d^2(f(h), f^\prime(h^\prime) | x^\pre) &= \frac{1}{n} \sum_{i=1}^n \left\|f(h(x_i^\pre)) - f^\prime(h^\prime(x_i^\pre)) \right\|^2\\
		&= \frac{1}{n} \sum_{s=1}^{k-1} \sum_{i=1}^n \left(f_s(h(x_i^\pre)) - f^\prime_s(h^\prime(x_i^\pre)) \right)^2
	\end{align*}
	Dudley's entropy integral bound~\citep[Theorem 5.22]{wainwright2019high} shows 
	\begin{align*}
		&\quad \E \sup_{f(h), f^\prime(h^\prime)} Z_{f(h)} - Z_{f^\prime(h^\prime)} \\
		&\le 2 \E \sup_{d(f(h), f^\prime(h^\prime) |x^\pre) \le \delta} Z_{f(h)} - Z_{f^\prime(h^\prime)} + 32\gJ(\frac{\delta}{4}, D_{\gX^\pre})\\
		&= 2 \E \sup_{d(f(h), f^\prime(h^\prime) | x^\pre) \le \delta} Z_{f(h)} - Z_{f^\prime(h^\prime)} + 32 \int_{\frac{\delta}{4}}^{D_{\gX^\pre}} \sqrt{\log N(u; \gF^\pre (\gH) | x^\pre)} du.
	\end{align*}
	It is straightforward to see the first term follows:
	\begin{equation*}
		\E \sup_{d(f(h), f^\prime(h^\prime) | x^\pre) \le \delta} Z_{f(h)} - Z_{f^\prime(h^\prime)} \le \E[\|g\|] \delta \le \sqrt{n(k-1)} \delta
	\end{equation*}
	
	We now turn to bound the second term by decomposing the distance metric into a distance over $\gF^\pre$ and a distance over $
	\gH$. We claim that, for arbitrary $h \in \gH, f \in \gF^\pre$, let $h^\prime$ be $\epsilon_1$-close to $h$ in empirical $l_2$-norm w.r.t inputs $x_1^\pre, x_2^\pre \cdots, x_n^\pre$. Given $h^\prime$, let $f^\prime$ be $\epsilon_2$-close to $f$ in empirical $l_2$ loss w.r.t $h^\prime(x^\pre)$. Using the triangle inequality we have that
	\begin{align*}
		d(f(h),f^\prime(h^\prime)| x^\pre) &= \sqrt{\frac{1}{n} \sum_{i=1}^n \|f(h(x_i^\pre)) - f^\prime(h^\prime(x_i^\pre))\|}\\
		&\le d(f(h),f(h^\prime)| x^\pre) + d(f(h^\prime),f^\prime(h^\prime) | x^\pre)\\
		&\le \sqrt{\frac{1}{n} \sum_{i=1}^n \|f(h(x_i^\pre))- f(h^\prime(x_i^\pre))\|^2} + \epsilon_2\\
		&\le L(\gF^\pre) \sqrt{\frac{1}{n} \sum_{i=1}^n \|h(x_i^\pre)- h^\prime(x_i^\pre)\|^2} + \epsilon_2\\
		&= L(\gF^\pre) \cdot \epsilon_1 + \epsilon_2,
	\end{align*}
	where we have used that $\|f(x) - f(y)\| \le L(\gF^\pre) \|x-y\|$ for any $f \in \gF^\pre$.
	
	As for the cardinality of the covering $C_{\gF^\pre (\gH)}$. Observe $|C_{\gF^\pre (\gH)}| = \sum_{h \in C_{\gH(x^\pre)}} |C_{\gF^\pre_h}| \le |C_{\gH(x^\pre)}| \cdot \max_{h \in \gH(x^\pre)} |C_{\gF^\pre_{h(x^\pre)}}|$. This provides a bound on the metric entropy of
	\begin{align*}
		\log N(\epsilon_1 \cdot L(\gF^\pre) + \epsilon_2; \gF^\pre (\gH)| x^\pre) \le \log N(\epsilon_1; \gH| x^\pre) + \max_{h(x^\pre)} N(\epsilon_2; \gF^\pre| h \circ x^\pre).
	\end{align*}
	Applying the covering number upper bound with $\epsilon_1 = \frac{\epsilon}{2 \cdot L(\gF^\pre)}, \epsilon_2 = \frac{\epsilon}{2}$ gives a bound of entropy integral ofa ,
	\begin{align*}
		& \quad \int_{\frac{\delta}{4}}^{D_{\gX^\pre}} \sqrt{\log N(u; \gF^\pre (\gH)| x^\pre)} du \\
		&\le \int_{\frac{\delta}{4}}^{D_{\gX^\pre}} \sqrt{\log N \left(\frac{u}{2 L(\gF^\pre)}; \gH \bigg{|} x^\pre\right)} du + \int_{\frac{\delta}{4}}^{D_{\gX^\pre}} \max_{h \circ x^\pre} \sqrt{\log N\left(\frac{u}{2}; \gF^\pre \big{|} h \circ x^\pre \right)} du
	\end{align*}
	From the Sudakov minoration theorem~\citep[Theorem 5.30]{wainwright2019high} for Gaussian processes and the fact that packing
	numbers at scale $u$ upper bounds the covering number at scale $\forall u >0$ we find:
	\begin{equation*}
		\log N(u; \gH| x^\pre) \le 4 \left(\frac{\sqrt{n}\hat{G}_{n}(\gH)}{u}\right)^2, \quad \log N(u; \gF^\pre| h(x^\pre)) \le 4 \left(\frac{\sqrt{n}\hat{G}_n(\gF^\pre| h \circ x^\pre)}{u}\right)^2.
	\end{equation*}
	Combining the definition of worst-case Gaussian complexity with the aforementioned results we have
	\begin{align*}
		\hat{G}_n(\gF^\pre \circ \gH) \le 2\sqrt{k-1} \delta + 256 \log{\frac{4D_{\gX^\pre}}{\delta}} \left(L(\gF^\pre) \hat{G}_n(\gH) + \bar{G}_n(\gF^\pre) \right),
	\end{align*}
     substitute $\delta$ with $\frac{4D_{\gX^\pre}}{n^2}$, proof is completed.
\end{proof}

\paragraph{Downstream learning}
Next we turn to the second stage and come up with theoretical guarantees by using inexact $\hat{h}$ learned from the first stage.
\begin{theorem} \label{thm_test_phase}
	In the downstream task, we have that with probability at least $1-\delta$,
	\begin{align*}
	R_\down( {\hat{f}}^\down, \hat{h} ) - R_\down(f^\down, h)  \le d_{\gF^\down}(\hat{h}; h) + 4\sqrt{\pi}L^\down \cdot \bar{G}_m(\gF^\down) + 4B^\down \sqrt{\frac{\log(\nicefrac{2}{\delta})}{m}}
	\end{align*}
\end{theorem}

\begin{proof} Assumption~\ref{assump_realizability} implies
\begin{equation*}
    \mathbb{E}_{x^{\down},y^{\down}} \left[\ell\left(g^\down\left(x^{\down}\right),y^{\down}\right)\right] = R_\down(f^\down, h).
\end{equation*}
To start, let $\tilde{f} = \argmin_{f \in \gF^\down} R_\down (f, \hat{h})$ and $R_\down( {\hat{f}}^\down, \hat{h} ) - R_\down(f^\down, h)$ equals
	\begin{align*}
		 \left[ R_\down(\tilde{f}, \hat{h}) - R_\down(f^\down, h) \right] + \left[R_\down( {\hat{f}}^\down, \hat{h} ) - R_\down(\tilde{f}, \hat{h}) \right]
	\end{align*}
	where the first term satisfies
	\begin{align*}
		& \qquad \inf_{\tilde{f} \in \gF^\down} \left[R_\down(\tilde{f}, \hat{h}) - R_\down(f^\down, h)\right] \\
		&\le \sup_{f^\down \in \gF^\down} \inf_{\tilde{f} \in \gF^\down} \big{[} R_\down(\tilde{f}, \hat{h}) - R_\down(f^\down, h) \big{]} \\
		&= d_{\gF^\down}(\hat{h}, h)
	\end{align*}
% 	\simon{Don't use $\cdots$, fill all details.}\yulai{fixed}
	
	The second term follows the similar lines of Theorem~\ref{thm_train_phase}
	\begin{align*}
		R_\down( {\hat{f}}^\down, \hat{h} ) - R_\down(\tilde{f}, \hat{h}) &\le 4\sqrt{\pi}L^\down \E_{\gX^\down} \hat{G}_m (\gF^\down | \hat{h} \circ x^{\down}) + 4B^\down \sqrt{\frac{\log{(\nicefrac{1}{\delta})}}{m}}
	\end{align*}
	Again we make use of the worst-case argument
	\begin{equation*}
		\E_{\gX^\down} \hat{G}_m (\gF^\down | \hat{h} \circ x^{\down} ) \le \bar{G}_m(\gF^\down).
	\end{equation*}
	Combining the results gives the statement.
\end{proof}

\paragraph{Proof of main Theorem~\ref{thm_main}} Having introduced class diversity parameter, proof is directly completed via combination of Theorem~\ref{thm_train_phase} and Theorem~\ref{thm_test_phase}.

\subsection{Proofs for Section~\ref{sec:nlp_case}}
\label{pf_cross_entropy}
We could provide a better dependence on the boundedness noise parameters in Theorem~\ref{thm_test_phase} using Bernstein inequality. We present the following corollary which has data-dependence in the Gaussian complexities.
\begin{corollary} \label{cor_test_phase}
Presuming Assumption~\ref{assump_realizability} holds, we have that then with probability at least $1-\delta$,
	\begin{align*}
		&\quad R_\down( {\hat{f}}^\down, \hat{h} ) - R_\down(f^\down, h)\\
		&\le d_{\gF^\down}(\hat{h}; h) + 4\sqrt{\pi}L^\down \cdot \E_{\gX^\down} \hat{G}_m (\gF^\down |\hat{h} \circ x^{\down} )+ 4\sigma \sqrt{\frac{\log(\nicefrac{2}{\delta})}{m}} + 50B^\down \frac{\log(\nicefrac{2}{\delta})}{m}
	\end{align*}
\end{corollary}
\begin{proof}
	Denote $Z = \sup_f |\hat{R}_\down(f, \hat{h}) - R_\down (f, \hat{h})| $, we apply the functional Bernstein inequality~\citep[Theorem 3]{massart2000constants} to control the fluctuations. With probability at lest $1-\delta$, we have
	\begin{align}
		Z \le 2\E[Z] + 4\frac{\sigma}{\sqrt{m}} \sqrt{\log(\frac{1}{\delta})} + 35 \frac{B^\down}{m} \log(\frac{1}{\delta}),
	\end{align} 
	where $\sigma^2 = \frac{1}{m} \sup_f \sum_{i=1}^m Var(\ell(f \circ \hat{h}(x_i^\down), y_i^\down))$. 
Thus
	\begin{align*}
		\E[Z] &\le 2 \E_{\gX^\down} \hat{R}_m (l( \gF^\down) | \hat{h} \circ x^\down)\\
		&\le 2 \E_{\gX^\down} \sqrt{2} L^\down \hat{R}_m (\gF^\down | \hat{h} \circ x^\down)\\
		&\le 2 \E_{\gX^\down} \sqrt{\pi} L^\down \hat{G}_m (\gF^\down | \hat{h} \circ x^\down),
	\end{align*}
	where the second line uses vector-based contraction principle, the last line upper bounds the empirical Rademacher complexity by Gaussian counterparts.
\end{proof}

\paragraph{Proof of Theorem~\ref{thm_cross_entropy}} \label{pf_thm_cross_entropy}
Observe that
$$\ell(\eta ; y) =  - y^\top \eta + \log{(1 + \sum_{s=1}^{k-1} e^{\eta_s})}, \ell(\eta ; y) \le \|\eta\|$$
and $$\left|\frac{\partial \ell(\eta;y)}{\partial \eta_i}\right| = \left|y_i - \frac{e^{\eta_i}}{1 + \sum_{s=1}^{k-1} e^{\eta_s}}\right|,$$ 
$$\left| \nabla_\eta \ell(\eta;y) \right| \le \sqrt{k-1},$$
so it is $L^\pre=\sqrt{k-1}-$Lipschitz. 
By definition the class $\gF^\pre$ with parameters $\|\alpha_s\|_2 \le O(1), s \in [k-1]$, we obtain that $L(\gF^\pre) = O \left(\sqrt{k-1}\right)$ since for any $x, y \in 
\sR^r$, any $f \in \gF^\pre$ we have
\begin{align*}
	\| f(x) - f(y) \|^2 &= \| \alpha^\top x - \alpha^\top y\|^2\\
	&\le \sum_{s=1}^{k-1} \left( \langle \alpha_s, x-y \rangle \right)^2\\
	&\le \sum_{s=1}^{k-1} \|\alpha_s\|^2 \|x-y\|^2 \\
	&\le c_1^2 (k-1) \|x-y\|^2
\end{align*}

In conclusion we have
\begin{itemize}
	\item Pre-training loss $\ell (\cdot, y^\pre)$ is $\sqrt{k - 1}$-Lipschitz.
	\item Downstream loss $\ell (\cdot, y^\down)$ is $\sqrt{k^\prime - 1}$-Lipschitz.
	\item Linear layer $f$ is $L(\gF^\pre) = O\left( \sqrt{k-1} \right)$-Lipschitz.
\end{itemize}
Consider task-specific function classes for characterizing class-diversity parameters. From Lemma~\ref{lem_quadratic_bound} and Lemma~\ref{lem_lowerbound_nu} we know that
\begin{equation*}
    \nu = \Omega(\tilde{\nu}), \quad \tilde{\nu} = \sigma_r (\alpha_1 \alpha_1^\top).
\end{equation*}
 Combining these pieces of results then the proof is completed.\\
 
With the following proposition, we interpret the cross-entropy loss in the well-specified model under our multinomial logistic model distribution.
\begin{proposition}\label{prop_inter}
Under Assumption~\ref{assump_samples}, for the cross entropy loss $\ell$ we have
\begin{align*}
    \E_{y \sim \gP(\cdot|f\circ h(x))} [\ell(\hat{f} \circ \hat{h}(x), y)]-\ell(f \circ h(x), y)] &= KL \left[\gP(\cdot| f \circ h(x)), \gP(\cdot| \hat{f}\circ \hat{h}(x)) \right]\\
    &= KL \left[\gP(\cdot| \alpha^\top h(x)), \gP(\cdot| {\alpha^\prime}^\top \hat{h}(x)) \right].
\end{align*}
\end{proposition}
Recall that $\alpha^\prime$ and $\alpha$ are parameters for $\hat{f}$ and $f$ respectively. The proof is straightforward by applying Assumption~\ref{assump_samples}.

\paragraph{Proof of Proposition~\ref{prop_self}}
\begin{proof}
	Let $P(t; v^0) = 1 + \sum_s e^{u_s + t v_s}$ and $P(t; v^i) = \sum_s v_s^i e^{u_s + t v_s}, i >1$. Then we use multinomials $P$ to represent derivatives of $g(t)$
	\begin{align*}
		g(t) &= \log(P(T; V^0))\\
		g'(t) &= \frac{P(t; v^1)}{P(t; v^0)}\\
		g''(t) &= \frac{P(t; v^2) P(t; v^0)}{P(t; v^0)^2}\\
		g'''(t) &= \frac{P(t; v^3) P(t; v^0)^2 - 3P(t; v^2) P(t; v^1) P(t; v^0) + 2 P(t; v^1)^3}{P(t; v^0)^3}
	\end{align*}
Let $r_s = e^{u_s + v_s t}$, hence
\begin{align*}
	g''(t) &= \frac{ (\sum_s v_s^2 r_s)\cdot (1 + \sum_s r_s) - (\sum_s v_s r_s)^2 }{( 1 + \sum_s r_s )^2}\\
	&= \frac{ \sum_{i<j} r_i r_j (v_i - v_j)^2 + \sum_i v_i^2 r_i }{( 1 + \sum_s r_s )^2}
\end{align*}
In the following we expand $g'''(t)$ as:
%we adapt proof in~\citep[Lemma 4]{tran2015composite}
\begin{align*}
	%& \frac{ \sum_{i<j} r_i r_j (v_i - v_j)^2 [\sum_k (v_i+v_j-2v_k)r_k] + \sum_i v_i^3 r_i + 2\sum_i \sum_j v_i^3 r_i r_j - 3 \sum_i v_i^2 r_i \sum_j v_j r_j }{( 1 + \sum_s r_s )^3}\\
	& \frac{ \sum_{i<j} r_i r_j (v_i - v_j)^2 [\sum_k (v_i+v_j-2v_k)r_k] + \sum_i v_i^3 r_i + \sum_i \sum_j v_i^2 r_i r_j(2v_i- 3v_j) }{( 1 + \sum_s r_s )^3}\\
	&=  \frac{ \sum_{i<j} r_i r_j (v_i - v_j)^2 [\sum_k (v_i+v_j-2v_k)r_k] + \sum_i v_i^2 r_i \left( v_i(1+2\sum_j r_j) - 3\sum_j v_j r_j \right) }{( 1 + \sum_s r_s )^3},
\end{align*}
observe that
\begin{gather*}
	\frac{1}{1 + \sum_s r_s} \left|\sum_k (v_i + v_j - 2v_k) r_k\right| \le \sum_k |v_i + v_j - 2v_k| \frac{r_k}{1 + \sum_s r_s} \le 4\|v\|_2\\
	\frac{1}{1 + \sum_s r_s} \left| v_i(1+2\sum_j r_j) - 3\sum_j v_j r_j \right| \le 5 \|v\|_2
\end{gather*}
Substitute these into definition of \emph{self-concordance} then proof is completed.
\end{proof}

Now we are ready to give a lower bound of $KL-$divergence,
\begin{align*}
	& \quad \Phi({\alpha^\prime}^\top \hat{h}(x)) - \Phi(\alpha^\top h(x)) - \nabla \Phi(\alpha^\top h(x))^\top v \\
	&\ge \frac{1}{2} v^\top e^{- 5 \|v\|_2} F''(\alpha^\top h(x)) v\\
	&\ge \frac{1}{2} \lambda_{min}(\Phi''(\alpha^\top h(x))) \|v\|^2 e^{- 5 \|v\|_2}\\
	&\ge \frac{1}{2} \lambda_{min}(\Phi''(\alpha^\top h(x))) \|v\|^2 e^{- 5(\|{\alpha^\prime}^\top \hat{h}(x)\| + \|\alpha^\top h(x)\| )}\\
	&\ge \frac{1}{2} \lambda_{min}(\Phi''(\alpha^\top h(x))) \|v\|^2 \exp( -10 q_0 )
\end{align*}
where $v = {\alpha^\prime}^\top \hat{h}(x) - \alpha^\top h(x)$.
Proof for Lemma~\ref{lem_quadratic_bound} is completed. \qed 

\paragraph{Proof of Lemma~\ref{lem_lowerbound_nu}}
\begin{proof}
	For function classes $\gF^\pre$, $\gF^\down$ and data samples generated from multinomial logistic regression distribution (see Assumption~\ref{assump_samples}), the worst-case representation difference is similar to that in multi-task analysis~\citep[Lemma 1]{tripuraneni2020theory}:
	
	\begin{align*}
		d_{\gF^\down}(\hat{h}; h) &= \sup_{f^\down \in \gF^\down} \inf_{f\prime \in \gF^\down} \E \left\{ \ell(f^\prime \circ \hat{h}(x^\down), y^\down) - \ell(f^\down \circ h(x^\down), y^\down) \right\} \\
		&\le \sup_{\|\alpha_s\| \le c_0} \inf_{\|\alpha^\prime_s\| \le c_0 } \frac{1}{2} \E_{\gX^\down} \left\| {\alpha^\prime}^\top \hat{h}(x^\down) - \alpha^\top h(x^\down) \right\|^2, \quad \text{here } s \in [k^\prime-1] \\
		&= \sum_{s=1}^{k^\prime-1} \sup_{\|\alpha_s\| \le c_0} \inf_{\|\alpha^\prime_s\| \le c_0} \frac{1}{2} \E_{\gX^\down} \left( {\alpha^\prime}_s^\top \hat{h}(x^\down) - \alpha_s^\top h(x^\down) \right)^2 \\
		&\le (k^\prime-1) \frac{c_0^2}{2} \sigma_1 (\Lambda_{sc}(\hat{h}, h)).
	\end{align*}
	The first line is because of Proposition~\ref{prop_inter} and Lemma~\ref{lem_quadratic_bound}. In the last line, the inner infima is considered as the partial minimization of a convex quadratic form (see \citep[Example 3.15, Appendix A.5.4]{boyd2004convex}).
	
Define population covariance if representations $\hat{h}$ and $h$ as
$$
\begin{gathered}\Lambda(\hat{h}, h) = 
\begin{bmatrix} \E [\hat{h}(x) \hat{h}(x)^\top] & \E [\hat{h}(x) h(x)^\top] \\ \E [h(x) \hat{h}(x)^\top] & \E [h(x) h(x)^\top] \end{bmatrix}
=
\begin{bmatrix} F_{\hat{h}\hat{h}} & F_{\hat{h}h} \\ F_{h \hat{h}} & F_{h h} \end{bmatrix}
\end{gathered}
$$
$\Lambda_{sc}(\hat{h}, h) = F_{h h} - F_{h \hat{h}} (F_{\hat{h} \hat{h}})^\dagger F_{\hat{h} h}$ is the generalized Schur complement of $h$ with respect to $\hat{h}$.

Next we control the pre-training representation differenceound is subtler,
	\begin{align*}
		&\quad d_{\gF^\pre, f^\pre}(\hat{h}; h) \\
		&\ge  \inf_{\alpha^\prime} c_0 \E_{\gX^\pre} \left[ \exp( - 10\max(\|{\alpha^\prime}^\top \hat{h}(x^\pre)\|, \| \alpha^\top h(x^\pre) \|) ) \cdot \left\| {\alpha^\prime}^\top \hat{h}(x^\pre) - \alpha^\top h(x^\pre) \right\|^2 \right],
	\end{align*}
which is because of Proposition~\ref{prop_inter} and Lemma~\ref{lem_quadratic_bound}.

It is known
\begin{align*}
	 &\quad \E_{\gX^\pre} \left[ \exp( - 10\max(\|{\alpha^\prime}^\top \hat{h}(x^\pre)\|, \| \alpha^\top h(x^\pre) \|) ) \cdot \left\| {\alpha^\prime}^\top \hat{h}(x^\pre) - \alpha^\top h(x^\pre) \right\|^2 \right] \\
		 &\ge e^{-10 c_2 }\left\| {\alpha^\prime}^\top \hat{h}(x^\pre) - \alpha^\top h(x^\pre) \right\|^2
	\end{align*}
Hence this metric could be claimed to be lower bounded as,
\begin{align*}
	&\quad \Omega\left( \inf_{\alpha^\prime} \E_{x^\pre} \left\| {\alpha^\prime}^\top \hat{h}(x^\pre) - \alpha^\top h(x^\pre) \right\|^2 \right)\\
	&=\Omega\left( \alpha_1^\top \Lambda_{sc}(\hat{h};h) \alpha_1 \right)\\
	&= \Omega\left( tr(\Lambda_{sc}(\hat{h}; h)C) \right), \quad \text{where $C = \alpha_1 \alpha_1^\top$.}
\end{align*}
In the second line, we redefine $\alpha_1$ as parameter $\alpha$ of pre-training for clarity.
In this way we conclude that, 
\begin{align*}
	d_{\gF^\pre, f^\pre}(\hat{h}; h) = \Omega\left( tr(\Lambda_{sc}(\hat{h}, h)C) \right) = \Omega\left( \sigma_1(\Lambda_{sc}(\hat{h}, h)) \sigma_r(C) \right),
\end{align*}
where $C$ implies expansion of representation $h(x) \in \sR^r$, and its condition number $\sigma_r(C)$ indicates how spread out  this vector is in $\sR^r$:
\begin{align*}
	C = \sum_{s=1}^{k-1} (\alpha_1)_s (\alpha_1)_s^\top = \alpha_1 \alpha_1^\top, \quad \alpha_1 \in \sR^{r \times (k-1)}
\end{align*}
Aforementioned calculations show
\begin{equation*}
	d_{\gF^\down}(\hat{h}; h) \le \frac{1}{\Omega(\tilde{\nu})} d_{\gF^\pre, f^\pre}(\hat{h}; h), \quad \tilde{\nu} = \sigma_r(C).
\end{equation*}
Proof is completed.
\end{proof}

%% file: pf_linear.tex
\subsection{Proofs for Section~\ref{sec:subspace}}\label{pf_app_1}
\begin{proof}
	We begin with bounding each of the complexity terms in the Corrolary~\ref{cor_test_phase}. 
	
	We make use of data-dependent inequalities~\citep[Lemma 4]{tripuraneni2020theory} to help upper bound related quantities. Intuitively Definition~\ref{defn_regularity_conditions} implies tail-bound properties in a sub-gaussian process.
	\begin{itemize}
		\item \begin{align*}
			\hat{G}_{n}(\gH) &= \frac{1}{n} \E \left[\sup_{B \in \gH} \sum_{k=1}^r \sum_{i=1}^n g_{ki} b_k^\top x_{i}^\pre \right]\\
			&= O \left( \sqrt{\frac{d r^2}{n}} \right)
		\end{align*}
		\item \begin{align*}
			\hat{G}_n(\gF^\pre| h \circ x^\pre ) &= \frac{1}{n} \E \left[\sup_{\alpha_1, \cdots, \alpha_{k-1}} \sum_{s=1}^{k-1} \sum_{i=1}^n g_{is} \alpha_s^\top B^\top x_{i}^\pre \right]\\
			&= \frac{c_1(k-1)}{n} \E \| \sum_{i=1}^n g_{is} B^\top x_{i}^\pre\| \\
			&= \frac{c_1 (k-1)}{\sqrt{n}} \sqrt{tr(B^\top \Sigma B)}
		\end{align*}
		then $\bar{G}_n(\gF^\pre) \le O\left((k-1)\sqrt{\frac{r}{n}} \right)$.
		\item Similarly, 
		\begin{align*}
			\hat{G}_m(\gF^\down | h \circ x^\down) \le \frac{c_1(k^\prime-1)}{\sqrt{m}} \sqrt{\sum_{i=1}^r \sigma_i(\hat{B}^\top \Sigma \hat{B})}
		\end{align*}
		then $\bar{G}_m(\gF^\down) \le O\left((k^\prime-1)\sqrt{\frac{r}{m}} \right)$.
		\item boundedness parameter $D_{\gX^\pre} = \sup_{\alpha, B} \| \alpha^\top B^\top x\| = c_2$
		\item cross entropy $\ell(\eta ; y) =  - y^\top \eta + \log{(1 + \sum_{s=1}^{k-1} e^{\eta_s})}$, then $\left| \frac{\partial \ell(\eta;y)}{\partial \eta_i} \right| = \left|y_i - \frac{e^{\eta_i}}{1 + \sum_{s=1}^{k-1} e^{\eta_s}}\right|$, $\left| \nabla_\eta \ell(\eta;y) \right| \le \sqrt{k-1}$, so it is $L^\pre =\sqrt{k-1}-$Lipschitz in its first coordinate uniformly over its second for pre-training and $L^\down = \sqrt{k^\prime-1}-$Lipschitz for downstream task. 
		\item $|\ell(\eta;y)| \le O(\|\eta\|)$, where $\|\eta\| = \| x^\top B^\pre \alpha\| \le c_2$. 
	\end{itemize}
	In Corollary~\ref{cor_test_phase}, we define and compute the maximal variance term $\sigma^2$ as,
	\begin{align*}
		\sigma^2 &= \frac{1}{m} \sup_{f^\down \in \gF^\down} \sum_{i=1}^m Var(\ell^\prime(f^\down \circ \hat{h}(x_{i}^\down), y_{i}^\down))\\
		&\le \frac{k^\prime-1}{m} \sup_{f^\down \in \gF^\down} \sum_{i=1}^m Var(f^\down \circ \hat{h}(x_{i}^\down))\\
		&= \frac{k^\prime-1}{m} \sup_{\|\alpha_s\|\le O(1)} \sum_{s=1}^{k^\prime-1} \sum_{i=1}^m Var(\alpha_s^\top \hat{B}^\top x_{i}^\down)\\
		&= \frac{(k^\prime-1)^2}{m} \sup_{\|\alpha_s\|\le O(1)} \sum_{i=1}^m (\alpha_s \hat{B})^\top \Sigma \hat{B} \alpha_s\\
		&= (k^\prime-1)^2 O(\|\hat{B} \Sigma \hat{B} \|_2)\\
		&= O \left((k^\prime-1)^2\right)
	\end{align*}
	
	With these results in hand, we are now prepared to apply Corollary~\ref{cor_test_phase}, w.p. at least $1-\delta$
	\begin{align*}
		&\quad R_\down( {\hat{f}}^\down, \hat{h} ) - R_\down(f^\down, h) \\
		&\le d_{\gF^\down}(\hat{h};h) + 4 \sqrt{\pi}L^\down \cdot \E_{\gX^\down} \hat{G}_m \left(\gF^\down\Big{|} \hat{h} \circ x^{\down}\right) + 4\sigma \sqrt{\frac{\log{(\nicefrac{2}{\delta})}}{m}} + 50B^\down \frac{\log{(\nicefrac{2}{\delta})}}{m}
	\end{align*}
	where $\hat{G}_m(\gF^\down | \hat{h} \circ x^\down)$ is defined in Theorem~\ref{thm_cross_entropy}.
	
	Thus $L^\down \cdot \E_{\gX^\down} \hat{G}_m \left(\gF^\down\Big{|} \hat{h} \circ x^{\down}\right) \le L^\down \bar{G}_m(\gF^\down) \le O((k^\prime-1)^{\frac{3}{2}} \sqrt{\frac{r}{m}})$, $\sigma \le O(k^\prime-1)$, and $B^\down \le O(\sqrt{k^\prime-1}D)$.
	Further, we obtain upper bound of worst-case representation difference by diversity parameter and adoption of Theorem~\ref{thm_train_phase}: w.p. at least $1-\delta$
	\begin{align*}
		&\quad d_{\gF^\down}(\hat{h}; h) \\
		&\le \frac{d_{\gF^{\pre}, f^{\pre}}(\hat{h};h)}{\nu}\\
		&\le \frac{1}{\nu}\left\{  4096L \left[ \log{(n)} \cdot [L(\gF^\pre)\cdot G_{n}(\gH)+\bar{G}_n(\gF^\pre)] + \frac{\sqrt{k-1} D_{\gX^\pre}}{n^2} \right] + 4B \sqrt{\frac{\log{(\nicefrac{2}{\delta})}}{n}}   \right\}\\
		&\lesssim \frac{1}{\nu}\left\{ \sqrt{k} \log{(n)} \left( \sqrt{\frac{k dr^2}{n}} + k \sqrt{\frac{r}{n}} \right)  +    \frac{k}{n^2} + \sqrt{\frac{\log{(\nicefrac{1}{\delta})}}{n}}  \right\}
	\end{align*}
	The last thing to consider for completing the proof for Theorem~\ref{thm1} is giving accurate characterization of diversity parameter $\nu$, which we leave for the next subsection.
\end{proof}

%% file: pf_DNN.tex
\subsection{Proofs for Section~\ref{sec:nn}}\label{pf_app_2}

\begin{proof}
In deep neural network, we first review complexity quantities. Adapted from Theorem 8~\citep{golowich2017size}, we have
\begin{align*}
	\hat{R}_n(\gN) &\leq \left( \frac{2}{n} \Pi_{p=1}^K M(p) \right) \sqrt{(K+1+\log d) \cdot \max_{j \in [d]} \sum_{i=1}^n x_{i,j}^2} \\
	&\leq \frac{2D \sqrt{K+1+\log d} \cdot \Pi_{p=1}^K M(p)}{\sqrt{n}}.
\end{align*}
where $x_{i,j}$ denotes the $j$-th coordinate of vector $x_i$.

Then we proceed to bound the Gaussian complexities for our deep neural network and prove Theorem~\ref{thm2}. Recall that under the conditions of the result we can use former results to verify the task diversity condition is satisfied with parameters $\Omega(\tilde{\nu})$ with $\tilde{\nu} = \sigma_{r}(\alpha_1 \alpha_1^\top) > 0$. We can see that $\|\E_{x} [\hat{h}(x) h^*(x)^\top \|_2 \leq \E{x} \|\hat{h}(x) h^*(x)\| \leq O(M(K)^2)$ using the norm bound from. Hence under this setting we can choose $c_1$ sufficiently large so that $c_1 M(K)^2 \gtrsim \frac{M(K)^2}{c} c_2$. The condition $M(K) \gtrsim 1$ in the theorem statement is simply used to clean up the final bound.

 In order to instantiate Theorem~\ref{thm_main} we begin by bounding each of the complexity terms in the expression.
\begin{itemize}
	\item For the feature learning complexity in the training phase, we leverage above results, then
	\begin{align*}
 			& \hat{G}_{n}(\gH)=\frac{1}{n} \E[ \sup_{\gW_K} \sum_{k=1}^r \sum_{i=1}^{n} g_{ki} h_k(x_{i}^\pre)] \leq \sum_{k=1}^r \hat{G}_{n} (h_k(x_{i}^\pre)) \\
 			&\leq \log(n) \cdot \sum_{k=1}^r \hat{R}_{n}{h_k(x_{i}^\pre)} \leq r \log(n) \frac{2D \sqrt{K+1+\log d} \cdot \Pi_{p=1}^{K} M(p)}{\sqrt{n}}. 
		\end{align*}
		 This also implies the population Gaussian complexity.
    \item  By definition the class $\gF$ as linear maps with parameters $\|\alpha_s\|_2 \leq c_1 M(K)^2, \forall s \in [k-1]$, we obtain that $L(\gF) = c_1 \sqrt{k-1} M(K)^2$.
		\item For the complexity of learning $\gF^\pre$ in the training phase we obtain,
		\begin{align*}
			&  \hat{G}_n(\gF^\pre | h \circ x^\pre) = \frac{1}{n} \E_{g}[\sup_{\alpha \in \gF} \sum_{s=1}^{k-1} \sum_{i=1}^n g_{is} \alpha_s^\top h(x_{i}^\pre) ] \lesssim  \frac{(k-1) M(K)^2}{n} \E_{g} [\| \sum_{i=1}^n g_{is} h(x_{1i})\|]   \\
 			& \lesssim \frac{(k-1) M(K)^2}{n} \sqrt{\sum_{i=1}^n \|h(x_{i}^\pre)\|^2}  \lesssim \frac{(k-1) M(K)^2}{\sqrt{n}} \max_i \|h(x_{i}^\pre)\|.
		\end{align*}
		For $tanh$ activation function, we simply have
		\begin{align*}
			\|h(x)\|^2 = \| W_K r_{K-1} \|_2^2 \le \|W_K\|_{\infty \to 2}^2,
		\end{align*} where $r_{K-1}$ denotes ourput of the $K-1$th layer,
		 $$ \|h(x)\| \le O (M(K)).$$ 
		 In conclusion we obtain
 		\begin{align*}
			\bar{G}_n(\gF^\pre) \leq O \left( \frac{(k-1) M(K)^3}{\sqrt{n}} \right). 		
		\end{align*}
	\item Similarly
	\begin{align*}
 			\hat{G}_m(\gF^\down | h \circ x^\down) \leq O \left( \frac{(k^\prime - 1)M(K)^3 }{\sqrt{m}}  \right) 
	\end{align*} 
Then for \textbf{Regularity conditions} we have
   \item  Boundedness parameter $D_{\gX^\pre} = \sup_{\alpha, h} \|\alpha^\top h(x^\pre)\| = c_2$.
   \item Pre-training loss is $L^\pre = \sqrt{k-1}$-Lipschitz and $B^\pre = c_2$-bounded.
   \item Downstream loss is $L^\down = \sqrt{k^\prime-1}$-Lipschitz and $B^\down=c_3$-bounded.
\end{itemize}

Assembling the previous complexity arguments shows the transfer learning risk is bounded by
 	\begin{align*}
		&  \lesssim \frac{L^\pre}{\tilde{\nu}} \left( \log(n) \left[ L(\gF^\pre) r \log(n) \frac{D \sqrt{K} \Pi_{p=1}^{K} M(p)}{\sqrt{n}} + \frac{k M(K)^3}{\sqrt{n}} \right] \right)+ \frac{L^\down k^\prime M(K)^3}{\sqrt{m}} \\
		& + \left(\frac{1}{\tilde{\nu}}  \max \left( \frac{L^\pre \sqrt{k} D_{\gX^\pre} }{n^2}, B^\pre \sqrt{\frac{\log(1/\delta)}{n}} \right)+ B^\down \sqrt{\frac{\log(1/\delta)}{m}}\right)
	\end{align*}
Substitute regularity conditions into it, then the risk is simplified as stated in Theorem~\ref{thm2}.
\end{proof}

%% file: exp_new.tex
\begin{table}[t] \label{tab:glue_benchmark}
	\caption{\textbf{Performance of diversity-regularized BERT pre-training with different values of diversity factor $\lambda$.} We finetune the pre-trained model on $8$ downstream tasks from GLUE benchmark and evaluate them on their dev sets. All results are ``mean (std)'' from 5 runs with different random seeds. For MNLI, we average the accuracies on its matched and mismatched dev sets. For MRPC and QQP, we average their accuracy and F1 scores. For STS-B, we average Pearson’s correlation and Spearman’s correlation. All other tasks uses accuracy as the metric. The better-than-baseline numbers are underlined, and the best numbers are highlighted in boldface.}
	\centering
	\resizebox{\columnwidth}{!}{%
		\begin{tabular}{l|cccccccc}
			\toprule
			Model   & MNLI & MRPC & SST-2 & CoLA & QQP & QNLI & RTE & STS-B\\
% 			\cmidrule(r){2-8}
			\midrule
			BERT-base ($\lambda=0.005$)  & \underline{\textbf{84.17}} (0.23) & \underline{87.16} (1.81) & 92.48 (0.19) & 59.99 (0.28) & \underline{89.42} (0.08) & \underline{\textbf{88.11}} (0.54) & \underline{67.28} (3.43) & \underline{89.33} (0.07)\\
			BERT-base ($\lambda=0.05$)  & \underline{84.01}	(0.10) & \underline{86.35} (5.15) & \underline{\textbf{93.00}} (0.16) & \underline{\textbf{62.66}} (1.07) & \underline{\textbf{89.46}} (0.03) & 87.64	(0.44) & 60.64 (6.08) & \underline{\textbf{89.57}} (0.13) \\
% 			BERT-base (diversity=0.2)  &  &  & &  &  &  &  \\
			BERT-base ($\lambda=0.5$)  & \underline{84.00} (0.20) & \underline{\textbf{89.42}} (0.51) & \underline{92.93} (0.24) & 60.76 (0.71) & \underline{89.33} (0.12) & 88.01 (0.23) & \underline{\textbf{67.93}} (1.18) & \underline{89.22} (0.23)\\
			\midrule
			BERT-base (reproduced) & 83.96 (0.08) & 86.14 (4.64) & 92.64 (0.20) & 61.46 (0.74) & 89.28 (0.09) & 88.10 (0.27) & 63.64 (6.64) & 89.19 (0.07) \\
			\bottomrule
		\end{tabular}%
	}
\end{table}

Our theoretical analysis in previous sections implies that the diversity of the model parameter matrix at the linear output layer in pre-training has a significant impact on the transfer capability, in the sense that the larger $\nu$ (diversity parameter of $f^\pre$), the smaller the risk.
Therefore, we could \emph{explicitly add a diversity regularizer} to the linear output layer to increase diversity. Motivated by this, we propose to add the following diversity regularizer to the original BERT pre-training loss so that it becomes:
    \begin{align}
        L'(\Theta)
            &=
                L(\Theta) - \lambda \cdot \ln \det (\alpha^{\pre} \left(\alpha^{\pre}\right)^\top),
    \end{align}
where $\Theta$ denotes the set of all model parameters, $\lambda$ is a hyper-parameter that controls the magnitude of the diversity regularization, $\det(\cdot)$ denote the determinant of a matrix, and $\alpha^{\pre}$ is the model parameter matrix at the output linear layer. This type of diversity regularizer was proposed in~\citet{zou2012priors}.
%...\jianshu{@Simon: could you add ref and some discussion here?} 
This regularization technique is different from prior work because it is specifically designed for multi-class pre-training: we only add the diversity regularizer to the last linear layer.

We use the above diversity-regularized loss (along with the original $\ell_2$-regularization) to pretrain BERT-base models under different values of diversity factor $\lambda$. 
Then we fine-tune them on $7$ classification tasks and $1$ regression task from the GLUE benchmark~\citep{wang2018glue} to evaluate their transfer performance.\footnote{We do not report the WNLI (classification) task due to its reported issues of the task in~\citet{devlin2018bert}.} Our pre-training and finetuning implementations are based on the opensource code released by Nvidia.\footnote{Distributed under Apache License: \url{https://github.com/NVIDIA/DeepLearningExamples/tree/master/PyTorch/LanguageModeling/BERT}} We use the same pre-training data as the original BERT (i.e., English Wikipedia + TorontoBookCorpus).\footnote{Collected and pre-processed using the code and script included in the open-source code: \url{https://github.com/NVIDIA/DeepLearningExamples/tree/master/PyTorch/LanguageModeling/BERT}} Our detailed pre-training and finetuning hyper-parameters along with other experimental details are reported in Appendix~\ref{appendix_exp}.

In Table~\ref{tab:glue_benchmark}, we report our performance on the dev sets of the 8 downstream tasks. All the experiments are repeated $5$ times with different random seeds, and we report their mean values along with the standard deviations. The complete experiment results (including full MNLI, QQP, and MRPC results) can be found in Appendix~\ref{appendix_exp}. 
From Table~\ref{tab:glue_benchmark}, we note that adding the diversity regularization could generally improve the performance on these downstream tasks. 
In particular, when $\lambda=0.5$, our pre-trained model outperforms the original BERT-base on $6$ out of $8$ tasks (with $3$ of them being significant), while achieving comparable performance on the other $2$ tasks. 
Although our model is slightly behind the original BERT on CoLA and QNLI, such a performance gap is not statistically significant. Besides, we also see that our model with $\lambda=0.5$ achieves a much more stable performance (i.e., smaller std) on tasks with scarce finetuning data ($<4$K samples in MRPC and RTE). 
Our results, albeit still preliminary, demonstrate the potential of such a simple diversity-regularizer.
It could be an effective and simple performance booster for any of the existing pre-trained NLP models (e.g., XLNet~\citep{yang2019xlnet}, RoBERTa~\citep{liu2019roberta}, ALBERT~\citep{lan2019albert}, etc) with negligible computation and implementation cost.
%\jianshu{@Yulai: add reference to these works}\yulai{added} 
We leave the development of the more advanced diversity regularizer as a future work.

%% file: exp_supply.tex
\subsection{More Details} \label{appendix_exp}

Full statistics (including matched and mismatched dev sets for MNLI, accuracy and F1 scores for MRPC and QQP, and (Pearson’s correlation + Spearman’s correlation)/2 for STS-B. All other tasks uses accuracy as the metric) could be found in Table~\ref{tab:glue_full}.

\begin{table}[!htbp]\label{tab:glue_full}
\caption{Full statistics on GLUE dev sets.}
\centering
	\resizebox{\columnwidth}{!}{
		\begin{tabular}{l|c|cccccccc}
			\toprule
			Model & Statistics & MNLI(m/mm) & MRPC(acc/F1) & SST-2 & CoLA & QQP(acc/F1) & QNLI & RTE & STS-B\\
			\midrule
			$\lambda = 0.005$ &mean & 83.96/84.37 & 84.90/89.42 & 92.48 & 59.99 & 90.96/87.88 & 88.11 & 67.28 & 89.33\\
			&std & 0.26/0.21    & 2.28/1.34   & 0.19  & 0.28  & 0.05/0.11   & 0.54  & 3.43 & 0.07\\
			\midrule
			$\lambda = 0.05$ &mean & 83.88/84.14 & 83.72/88.98 & 93.00 & 62.66 & 90.97/87.96 & 87.64 & 60.64 & 89.57 \\
			&std & 0.04/0.16    & 6.69/3.62   & 0.16  & 1.07  & 0.05/0.04   & 0.44  & 6.08 & 0.13\\
			\midrule
			$\lambda = 0.5$	&mean & 83.96/84.04 & 87.75/91.09 & 92.93 & 60.76 & 90.85/87.81 & 88.01 & 67.93 & 89.22\\
			&std & 0.15/0.24    & 0.52/0.50   & 0.24  & 0.71  & 0.10/0.14   & 0.23  & 1.18 &0.23\\
			\midrule
			 BERT-base &mean & 83.85/84.07 & 83.48/88.80 & 92.64 & 61.46 & 90.87/87.68 & 88.10 & 63.64 & 89.19\\
			&std & 0.13/0.04    & 6.08/3.19   & 0.20  & 0.74  & 0.07/0.11   & 0.27  & 6.64 & 0.07\\
			\bottomrule
		\end{tabular}
		}
\end{table}

~\\
\paragraph{Complete statistics}
Here we provide complete results on GLUE dev sets over 5 random seeds.
\begin{table}[!htbp]
\small
	\caption{Performance of reproduced BERT-base model.}
	\centering
		\begin{tabular}{l|c|cccccccc}
			\toprule
			&GLUE   & MNLI(m/mm) & MRPC(acc/F1) & SST-2 & CoLA & QQP(acc/F1) & QNLI & RTE & STS-B\\
			\midrule
			&42  & 83.90/84.04 & 71.32/82.44 & 92.66 & 60.11 & 90.94/87.72 & 87.58 & 66.79 & 89.25 \\
			&0  & 83.86/84.12 & 86.27/90.18 & 92.43 & 62.05 & 90.74/87.53 & 88.19 & 54.29 & 89.07\\
			seed& 393  & 83.78/84.06 & 86.76/90.63 & 92.54 & 61.42 & 90.93/87.84 & 88.29 & 70.36 & 89.19\\
			&78 & 84.05/84.02 & 86.76/90.63 & 92.55 & 61.50 & 90.89/87.60 & 88.12 & 57.14 & 89.18\\
			&3837 & 83.66/84.11 & 86.27/90.13 & 93.00 & 62.20 & 90.87/87.73 & 88.33 & 69.64 & 89.26\\
			\midrule
			&mean & 83.85/84.07 & 83.48/88.80 & 92.64 & 61.46 & 90.87/87.68 & 88.10 & 63.64 & 89.19\\
			&std & 0.13/0.04    & 6.08/3.19   & 0.20  & 0.74  & 0.07/0.11   & 0.27  & 6.64 & 0.07\\
			\bottomrule
		\end{tabular}%
\end{table}

\begin{table}[!htbp]
\small
	\caption{Performance of $\lambda=0.005$ regularized pre-training model.}
	\centering
		\begin{tabular}{l|c|cccccccc}
			\toprule
			&GLUE  & MNLI(m/mm) & MRPC(acc/F1) & SST-2 & CoLA & QQP(acc/F1) & QNLI & RTE & STS-B\\
			\midrule
			&42  & 84.24/84.43 & 87.25/90.72 & 92.20 & 59.99 & 90.94/87.85 & 87.09 & 67.14 & 89.40\\
			&0  & 83.91/83.96 & 86.27/90.47 & 92.43 & 60.06 & 91.03/87.88 & 88.24 & 70.71 & 89.36\\
			seed &393 & 83.84/84.51 & 85.54/89.52 & 92.55 & 59.48 & 90.89/87.68 & 88.33 & 71.07 & 89.22 \\
			&78 & 84.23/84.44 & 84.80/89.45 & 92.78 & 60.13 & 90.95/87.97 & 88.71 & 65.71 & 89.38\\
			&3837 & 83.56/84.53 & 80.64/86.93 & 92.43 & 60.30 & 91.01/88.00 & 88.17 & 61.79 & 89.28\\
			\midrule
			&mean & 83.96/84.37 & 84.90/89.42 & 92.48 & 59.99 & 90.96/87.88 & 88.11 & 67.28 & 89.33 \\
			&std & 0.26/0.21    & 2.28/1.34   & 0.19  & 0.28  & 0.05/0.11   & 0.54  & 3.43 & 0.07\\
			\bottomrule
		\end{tabular}%
\end{table}

\begin{table}[!htbp]
\small
	\caption{Performance of $\lambda=0.05$ regularized pre-training model.}
	\centering
		\begin{tabular}{l|c|cccccccc}
			\toprule
			&GLUE   & MNLI(m/mm) & MRPC(acc/F1) & SST-2 & CoLA & QQP(acc/F1) & QNLI & RTE & STS-B\\
			\midrule
			&42  & 83.88/84.22 & 86.52/90.27 & 92.89 & 61.12 & 91.06/87.95 & 86.93 & 65.00 & 89.52 \\
			&0  & 83.83/83.92 & 88.97/92.00 & 93.00 & 64.36 & 90.96/87.93 & 87.73 & 54.29 & 89.65 \\
			seed&393  & 83.96/83.98 & 70.59/81.92 & 93.12 & 62.22 & 90.94/88.03 & 87.47 & 61.79 & 89.65 \\
			&78 & 83.86/84.26 & 87.50/91.06 & 92.78 & 63.13 & 90.94/87.97 & 88.26 & 68.93 & 89.69\\
			&3837 & 83.86/84.30 & 85.04/89.66 & 93.23 & 62.49 & 90.93/87.91 & 87.82 & 53.21 & 89.33 \\
			\midrule
			&mean & 83.88/84.14 & 83.72/88.98 & 93.00 & 62.66 & 90.97/87.96 & 87.64 & 60.64 & 89.57\\
			&std & 0.04/0.16    & 6.69/3.62   & 0.16  & 1.07  & 0.05/0.04   & 0.44  & 6.08 & 0.13\\
			\bottomrule
		\end{tabular}%
\end{table}

\begin{table}[!htbp]
\small
	\caption{Performance of $\lambda=0.5$ regularized pre-training model.}
	\centering
		\begin{tabular}{l|c|cccccccc}
			\toprule
			&GLUE   & MNLI(m/mm) & MRPC(acc/F1) & SST-2 & CoLA & QQP(acc/F1) & QNLI & RTE & STS-B\\
			\midrule
			&42  & 83.75/83.87 & 87.75/90.89 & 93.00 & 59.79 & 90.88/87.75 & 87.58 & 65.71 & 89.00 \\
			&0  & 83.84/84.30 & 88.24/91.56 & 92.66 & 60.94 & 90.87/87.77 & 88.14 & 67.86 & 89.25 \\
			seed&393  & 83.98/83.68 & 87.99/91.46 & 93.12 & 60.99 & 90.98/88.04 & 88.22 & 68.21 & 89.19\\
			&78 & 84.02/84.29 & 87.99/91.33 & 93.23 & 60.22 & 90.87/87.86 & 88.12 & 68.93 & 89.65\\
			&3837 & 84.19/84.05 & 86.76/90.21 & 92.66 & 61.86 & 90.67/87.61 & 87.98 & 68.93 & 89.03\\
			\midrule
			&mean & 83.96/84.04 & 87.75/91.09 & 92.93 & 60.76 & 90.85/87.81 & 88.01 & 67.93& 89.22 \\
			&std & 0.15/0.24    & 0.52/0.50   & 0.24  & 0.71  & 0.10/0.14   & 0.23  & 1.18& 0.23\\
			\bottomrule
		\end{tabular}%
\end{table}

\newpage

Finally, we report detailed hyperparameter settings below. 
\paragraph{Pre-training} Hyperparameters for pre-training are shown in Table~\ref{tab:pretraining_hyperparams}. 

\begin{table}[!htbp] 
\small
\begin{center}
\begin{tabular}{lccc}
\toprule
\textbf{Hyperparam}  & \textbf{phase-1} & \textbf{phase-2} \\
\midrule 
Number of Layers & 12 & 12 \\
Hidden size & 768 & 768 \\
FFN inner hidden size & 3072 & 3072 \\
Attention heads & 12 & 12 \\
Steps & 7038 & 1563 \\
Optimizer & LAMB & LAMB \\
Learning Rate & 9e-3 & 6e-3 \\
$\beta_1$ & 0.9 & 0.9 \\
$\beta_2$ & 0.999 & 0.999 \\
WarmUp & 28.43 \% & 12.80 \%\\
Batch Size & 65536 & 32768\\
\bottomrule
\end{tabular}
\end{center}
\caption{
Hyperparameters used in pre-training our models. We use the LAMB optimizer~\citep{you2019large} for large-batch pretraining of the BERT model, where $\beta_1$ and $\beta_2$ are its two hyper-parameters. 
}
\label{tab:pretraining_hyperparams}
\end{table}

\paragraph{Finetuning} Hyperparameters for downstream tasks are shown in Table~\ref{tab:downstream-hyperparameter}.
We adapt these hyperparameters from \cite{liu2019roberta},  \cite{devlin2018bert}, and \cite{yang2019xlnet}.

\begin{table}[!htbp] 
	\small
	\centering
	
	\begin{tabular}{c|ccccccc}
		& LR & BSZ & $\#$ EP & WARMUP & WD & FP16 & SEQ \\\hline
		CoLA & 1.00E-05 & 32 & 20 & 6\% & 0.1 & O2 & 128 \\
		SST-2 & 3.00E-05 & 32 & 10 & 6\% & 0.1 & O2 & 128 \\
		MNLI & 3.00E-05 & 32 & 5 & 6\% & 0.1 & O2 & 128 \\
		QNLI & 3.00E-05 & 32 & 10 & 6\% & 0.1 & O2 & 128 \\
		QQP & 3.00E-05 & 32 & 5 & 6\% & 0.1 & O2 & 128 \\
		RTE & 3.00E-05 & 16 & 5 & 6\% & 0.1 & O2 & 128 \\
		MRPC & 3.00E-05 & 16 & 5 & 6\% & 0.1 & O2 & 128
	\end{tabular}
	\caption{The hyperparameters used in finetuning our model in downstream tasks. LR: learning rate. BSZ: batch size. $\#$EP: number of epochs. WARMUP: warmup ratio.
		FP16: automatic mixed precision (AMP) level. SEQ: input sequence length.}
	\label{tab:downstream-hyperparameter}
\end{table}

\paragraph{Computing infrastructure} We pretrain our (diversity-regularized) BERT-base models using $32$ Nvidia V100 GPUs ($32$GB RAM each), and the finetuning of the model uses $4$ Nvidia V100 GPUs.